\documentclass{article}
\usepackage[preprint]{neurips_2026}

\usepackage[utf8]{inputenc} % allow utf-8 input
\usepackage[T1]{fontenc}    % use 8-bit T1 fonts
\usepackage{hyperref}       % hyperlinks
\usepackage{url}            % simple URL typesetting
\usepackage{booktabs}       % professional-quality tables
\usepackage{amsfonts}       % blackboard math symbols
\usepackage{nicefrac}       % compact symbols for 1/2, etc.
\usepackage{microtype}      % microtypography
\usepackage{xcolor}         % colors
 \usepackage{amsmath}
 \usepackage{amssymb}
 \usepackage{amsthm}
 \usepackage{graphicx}
 \usepackage{float}
 \usepackage[table]{xcolor}
 \usepackage{multirow}
 \usepackage{subcaption}

\theoremstyle{plain}
\newtheorem{theorem}{Theorem}
\newtheorem{corollary}[theorem]{Corollary}

\theoremstyle{definition}

\newtheorem{assumption}{Assumption}

\theoremstyle{remark}
\newtheorem*{remark}{Remark}
\usepackage{algorithm}% http://ctan.org/pkg/algorithms
\usepackage{algpseudocode}% http://ctan.org/pkg/algorithmicx
\usepackage{listings}
\usepackage[most]{tcolorbox}

\newtcblisting{promptbox}[1]{%
  enhanced, breakable,
  colback=white, colframe=blue!70!black,
  coltitle=white, colbacktitle=blue!70!black,
  fonttitle=\bfseries, title={#1},
  boxrule=1pt, arc=2mm,
  left=2mm, right=2mm, top=1mm, bottom=1mm,
  before skip=6pt, after skip=6pt,
  listing only,
  listing options={
    basicstyle=\ttfamily\small,
    breaklines=true,
    breakatwhitespace=false,
    columns=fullflexible,
    keepspaces=true,
    showstringspaces=false,
  }
}

\title{CoCoDA: Co-evolving Compositional DAG for Tool-Augmented Agents}

\author{%
\begin{tabular}{c}
\textbf{%
Ziyang Yu\textsuperscript{1} \quad
Qiyue Li\textsuperscript{2} \quad
Liang Zhao\textsuperscript{1}
}
\\[0.5em]
\normalfont
\textsuperscript{1}Emory University \qquad
\textsuperscript{2}Beijing Normal-Hong Kong Baptist University
\end{tabular}
}

\begin{document}

\maketitle

\begin{abstract}

Tool-augmented language models can extend small language models with external executable skills, but scaling the tool library creates a coupled challenge: the library must evolve with the planner as new reusable subroutines emerge, while retrieval from the growing library must remain within a fixed context budget. Existing tool-use and skill-library methods typically treat tools as flat or text-indexed memories, causing prompt cost to grow with library size and obscuring the typed, compositional structure of executable code. We propose CoCoDA, a framework that co-evolves the planner and tool library through a single code-native structure: a compositional code DAG. Nodes are primitive or composite tools, edges encode invocation dependencies, and each node stores a typed signature, description, pre/post-condition specification, and worked examples. At inference time, Typed DAG Retrieval prunes candidates by symbolic signature unification, ranks survivors by descriptions, filters them by behavioral specifications, and disambiguates with examples, keeping expensive context materialization on progressively smaller candidate sets. At training time, successful trajectories are folded into validated composite tools, while the planner is updated with a DAG-induced reward that credits composites by their primitive expansion size. We provide theoretical results showing retrieval cost reduction, sublinear retrieval time, compositional advantage under the shaped reward, monotone co-evolution under conservative updates, and DAG well-formedness. Across mathematical reasoning, tabular analysis, and code task benchmarks, CoCoDA enables an 8B student to match or exceed a 32B teacher on GSM8K and MATH and consistently improves over strong tool-use and library-learning baselines.

\end{abstract}
\vspace{-1em}
\section{Introduction}
\vspace{-0.7em}
Tool-augmented language models extend a planner with external executable tools, allowing even small language models to offload computation, symbolic manipulation, and long-horizon procedures to a reusable library. This division of labor is especially attractive for small language models, whose context windows and parametric capacity are limited. In principle, a sufficiently rich tool library can act as an external procedural memory: the planner retrieves a small set of relevant tools, invokes them through an executor, and composes their outputs into a final answer.

However, scaling such a library creates a tension. On one hand, the library must evolve with the planner. As the planner explores new tasks, it discovers recurring sub-trajectories that should be abstracted into reusable tools; otherwise, the planner repeatedly reconstructs the same primitive computations. On the other hand, every newly added tool increases the retrieval burden. A flat library makes the planner's prompt cost grow with the number of tools, eventually exhausting the context budget before the library becomes useful. Thus, tool-library learning for small language models requires solving two coupled problems: how to grow the library with the policy, and how to retrieve from the growing library under a fixed context window.

Existing approaches address only parts of this problem. Static tool-use systems assume a fixed inventory and train or prompt the model to call tools from that inventory. Skill-library and agent-memory methods append new skills during exploration, but typically store them as flat text records or natural-language memories, so retrieval cost grows with library size and the internal dependency structure of executable code remains hidden. Hierarchical memory and RAG methods reduce context cost by clustering summaries, but their hierarchies are induced by topical similarity rather than by executable composition. Conversely, code-library learning and program-synthesis methods exploit compositional structure, but they are not designed for LLM planners operating under token-budgeted retrieval. The missing mechanism is a structure that is simultaneously code-aware, retrieval-efficient, and learnable online.

Our key observation is that executable tools contain structure that text memories do not. A tool has a typed signature, executable dependencies, behavioral specifications, and concrete input-output examples. These signals can be used not only to organize the library, but also to decide which records must be exposed to the planner. We therefore represent the library as a compositional code DAG. Nodes are primitive or composite tools, edges are invocation dependencies extracted from executable bodies, and each node stores a four-level record: typed signature, description, pre/post-condition specification, and worked examples. This DAG is the single object through which CoCoDA couples library growth and context-bounded retrieval.

At inference time, CoCoDA performs Typed DAG Retrieval. Given a query-induced subgoal, retrieval first prunes candidates by symbolic signature unification, then ranks the survivors by descriptions, filters them by pre/post-condition compatibility, and finally uses examples for disambiguation. Because expensive LLM-billed records are materialized only for surviving candidates, the retrieval cost contracts across stages rather than scaling with the full library. Because the hierarchy follows invocation dependencies rather than text similarity, retrieval can surface a composite tool that behaviorally subsumes a primitive sub-trajectory.

At training time, CoCoDA co-evolves the planner and the library. Successful trajectories are passed to a fixed teacher abstractor, which proposes composite tools. A proposal is committed only if it preserves acyclicity and satisfies dependency and specification checks. The planner is updated with a DAG-induced reward that credits tools by their primitive expansion size: invoking a composite that represents multiple primitive calls receives positive saved-call credit, while invoking a primitive receives none. Thus, among behavior-equivalent trajectories with the same verified outcome, the planner is encouraged to use reusable composites. The same DAG therefore plays two roles: it keeps retrieval within the context budget and defines the structural signal that trains the planner to exploit newly learned abstractions.

We make the following contributions. First, we formulate tool-library learning for small language models as a joint planner-library optimization problem with explicit retrieval-cost and context-window constraints. Second, we propose Typed DAG Retrieval, a code-native retrieval cascade that uses signatures, invocation edges, specifications, and examples to keep prompt cost sublinear in library size. Third, we introduce an online co-evolution procedure that inserts validated composite tools and trains the planner with a DAG-induced saved-call reward. Fourth, we provide theoretical results showing retrieval cost reduction, sublinear retrieval time, compositional advantage under the shaped reward, monotone co-evolution under conservative updates, and DAG well-formedness. Finally, across mathematical reasoning, tabular analysis, and code task benchmarks, CoCoDA enables an 8B student to match or exceed a 32B teacher on GSM8K and MATH and consistently improves over strong tool-use and library-learning baselines.
\vspace{-1em}
\section{Related Works}
\vspace{-0.5em}
\subsection{Tool-Augmented Language Models}
\vspace{-0.5em}
Toolformer~\citep{toolformer2023} teaches an LLM to self-annotate API
calls and interleave them with generation, while
ReAct~\citep{react2023} threads reasoning traces with tool invocations
at inference time. ToolLLM~\citep{toolllm2024},
Gorilla~\citep{gorilla2024}, and AnyTool~\citep{anytool2024} scale tool
use to thousands of real-world APIs through instruction tuning and
hierarchical API retrieval. ToolAlpaca~\citep{toolalpaca2023} and
Gorilla~\citep{gorilla2024} distill tool-use trajectories from a
stronger teacher into smaller planners. RLTF~\citep{rltf2024},
ToRL~\citep{torl2024}, and ToolRL~\citep{toolrl2025} train LLMs
end-to-end with outcome rewards from execution feedback, and
ReTool~\citep{retool2025} and ARTIST~\citep{artist2025} extend the
recipe to multi-turn agentic settings on top of
GRPO~\citep{deepseekmath2024,deepseekr1}. Across these works, the tool
inventory is treated as a static, flat \emph{text-indexed} resource that
the policy optimizes against rather than reshapes, leaving both
co-evolution and the per-prompt scaling of retrieval cost unaddressed.
\vspace{-0.5em}
\subsection{Skill Library Learning and Agent Memory}
\vspace{-0.5em}
Prior agent-memory work treats the library as a flat \emph{text}
collection: Voyager~\citep{voyager2023}, Ghost in the
Minecraft~\citep{zhu2023ghost}, CRAFT~\citep{yuan2024craft},
LATM~\citep{latm2023}, CREATOR~\citep{creator2023},
TroVE~\citep{trove2024}, and ReGAL~\citep{regal2024} grow skill or
function memories offline whose dependency structure is invisible to
both retrieval and training, while
MemGPT~\citep{packer2024memgpt}, Generative
Agents~\citep{park2023generative}, Reflexion~\citep{shinn2023reflexion},
and Self-Refine~\citep{madaan2023selfrefine} hierarchize or self-refine
memory by topical similarity over natural language rather than by code
structure. A parallel line that does organize code into hierarchies
splits along an axis nothing has crossed: a hierarchy is either
\emph{code-aware} or \emph{built for an LLM under a fixed context
budget}, never both. DreamCoder~\citep{dreamcoder2021},
Stitch~\citep{stitch2023}, and software-engineering call-graph indexing
yield genuinely code-aware hierarchies, typed, compositional,
statically extractable, but target symbolic synthesizers or human
developers, with no token-budget tiering, no signature-level prefilter,
and no certified macro-action equivalence; conversely,
CodeRAG~\citep{coderag2024} and CodeT5+~\citep{codet5plus2023}
retrievers target LLMs but use a flat embedding space, and
RAPTOR~\citep{raptor2024} tiers for LLM context budgets yet clusters by
text similarity, discarding every code-specific signal. Our framework
occupies the missing quadrant via three properties that serve both
halves at once: (i)~\emph{typed signatures} are statically extractable
(code-aware) and prune candidates symbolically before any LLM call
(LLM-aware); (ii)~\emph{static invocation edges} follow real reuse
structure while keeping per-step cost proportional to local fan-out
rather than library size; and (iii)~\emph{pre/post-condition-validated}
composites preserve executable behavior and license \emph{lossless}
substitution of one composite token for a subtree of primitives in the
prompt, making context cost provably sublinear in library size where
prior code hierarchies remain worst-case linear.
\vspace{-0.5em}
\subsection{Policy--Memory Co-Evolution}
\vspace{-0.5em}
Voyager~\citep{voyager2023} is the canonical co-evolution setup,
appending new skills to a shared library as a frozen LLM explores, and
Experiential Co-Learning~\citep{qian2024experiential} extends the idea
to multi-agent software development. CREATOR~\citep{creator2023} and
LATM~\citep{latm2023} interleave tool creation with tool use, while
CRAFT~\citep{yuan2024craft} specializes toolsets to tasks encountered
during deployment. These works establish \emph{prompt-space, text-memory} co-evolution
where the library grows by text appends against a frozen or lightly
instructed large model and prompt cost grows linearly with library size,
rather than gradient-based co-evolution against a verifiable reward over
a structured \emph{code} library whose hierarchy keeps retrieval cost
sublinear.
\vspace{-1em}
\section{Proposed Method}
\vspace{-0.5em}
\label{sec:method}
\subsection{Problem Formulation}
\label{sec:problem}
\vspace{-0.5em}
\textbf{Setup.}
We study tool-mediated problem solving for a planner $\pi_\theta$ on
tasks $(q,y)\sim\mathcal{D}$, with $q$ a problem instance and $y$ a
verifiable outcome (final answer or unit-test oracle). The planner
accesses a tool library $\mathcal{L}=(\mathcal{V},\mathcal{E},\mathcal{I})$
where $\mathcal{V}$ is the set of tools, $\mathcal{E}$ records
invocation dependencies, and $\mathcal{I}$ stores per-tool records. Given $q$,
retrieval builds a candidate set
$\mathcal{C}_q(\mathcal{L})\subseteq\mathcal{V}$ that augments the
planner's prompt; the planner then samples
$\tau\sim\pi_\theta(\cdot\mid q,\mathcal{C}_q(\mathcal{L}))$ invoking
tools $(v_1,\ldots,v_T)$ via an executor. Quality is measured by a
deterministic verifier $R_{\text{res}}(\tau)\in[0,1]$, and the
context-token cost of $\mathcal{C}_q$ by $C(q,\mathcal{L})$.

\textbf{Compositional Tool DAG.}
We organise $\mathcal{L}$ as a DAG
$\mathcal{G}=(\mathcal{V},\mathcal{E})$, where
$\mathcal{V}=\mathcal{V}_p\cup\mathcal{V}_c$ partitions tools into
primitive nodes (atomic executable functions) and composite nodes
(macro-actions that invoke one or more existing tools). Each $v$
carries $(d(v),\mathrm{Ch}(v),\mathcal{I}(v))$: depth $d(v)=0$ for primitives
and $d(v)=1+\max_{u\in\mathrm{Ch}(v)}d(u)$ for composites; an ordered
child set $\mathrm{Ch}(v)$ realising $v$ at the next finer level; and a
record $\mathcal{I}(v)=(L_1,L_2,L_3,L_4)(v)$ giving a typed signature,
description, pre/post specification, and worked examples. Edges encode
invocation, so $\mathcal{G}$ stratifies the library by abstraction
rather than surface similarity. We further define the \emph{flat size}
$\mathrm{flat}(v)=1$ for primitives and
$\mathrm{flat}(v)=\sum_{u\in\mathrm{Ch}(v)}\mathrm{flat}(u)$ for
composites (counted with multiplicity: a child invoked $k$ times
contributes $k\,\mathrm{flat}(u)$), i.e.\ the number of primitive
leaves in the recursive expansion of $v$, and the
\emph{saved-call count} $\Phi(v):=\mathrm{flat}(v)-1$. Thus $\Phi=0$
on primitives and $\Phi(v)\ge 1$ on any composite expanding to $\ge 2$
primitive invocations.

{\setlength{\intextsep}{4pt}%        % 算法与正文之间的间距
\setlength{\floatsep}{4pt}% 
\begin{figure}[ht]
\centerline{\includegraphics[width=\columnwidth]{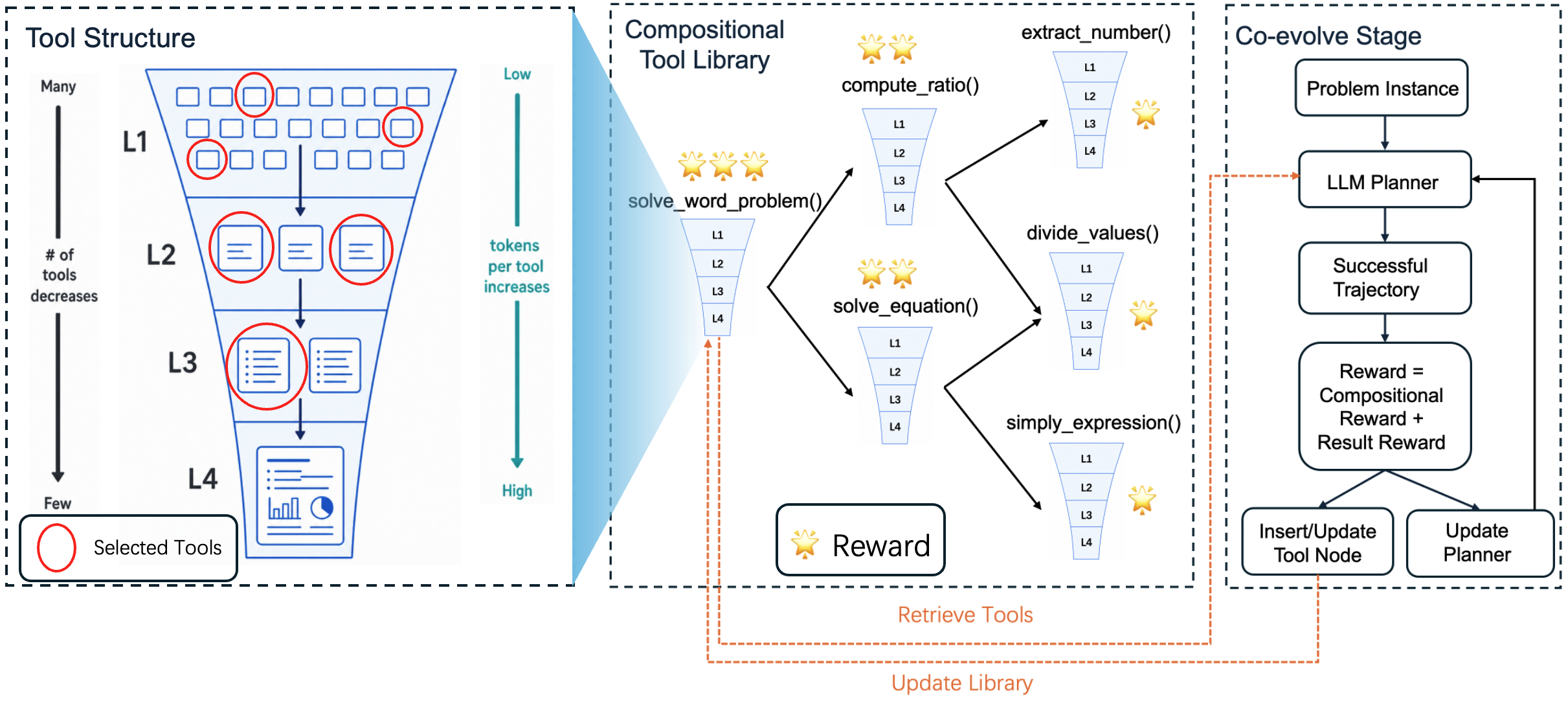}}
    \caption{Overview of the \textit{CoCoDA} framework.
    \textbf{Left (Tool Structure):} each tool stores a $4$-layer record ($L_1$ signature, $L_2$ description, $L_3$ specifications, $L_4$ examples), used by Typed DAG Retrieval as a cascade filter from cheap to expensive layers.
    \textbf{Middle (Compositional Tool Library):} a code DAG of primitive and composite tools (e.g., \texttt{compute\_ratio}, \texttt{solve\_equation}) whose edges encode invocation dependencies, letting macro-actions decompose into reusable subroutines.
    \textbf{Right (Co-evolve Stage):} the planner rolls out trajectories on the DAG; successful ones are abstracted into new composite tools via \textsc{Insert/Update Tool Node}, while the planner is updated with a compositional plus result reward, co-evolving library and policy.}
    \label{fig:framework}
\end{figure}}

\textbf{Tool-mediated output.}
We co-maintain $\pi_\theta$ (consumes $\mathcal{C}_q$, emits $\tau$)
and $\mathcal{L}$ (read by retrieval, written by \textsc{InsertTool},
which admits composites distilled from successful rollouts). We write
the joint output $O_{\pi_\theta,\mathcal{L}}(q)$ to make this explicit.
A fixed teacher abstractor $\mathcal{M}_T$ proposes candidates;
\textsc{InsertTool} accepts only those preserving acyclicity and
dependency specifications.

\textbf{Optimization objective.}
We frame \textit{CoCoDA} as a single optimization over the
planner--library pair in terms of quantities intrinsic to the LLM. Let
$\mathcal{D}_{\text{ev}}$ be the training distribution. For
$\tau\sim\pi_\theta(\cdot\mid q,\mathcal{C}_q(\mathcal{L}))$ with
$T(\tau)$ tool calls, let $\bar C_{\text{retr}}(q,\mathcal{L})$ be the
average per-call retrieval token cost; total inference token cost
factorises as $T(\tau)\cdot\bar C_{\text{retr}}(q,\mathcal{L})$. We seek
\vspace{-0.5em}
\begin{equation}
\label{eq:objective}
\begin{aligned}
(\pi_{\theta^\star},\mathcal{L}^\star)
\;&=\;
\mathop{\arg\max}\nolimits_{\pi_\theta,\,\mathcal{L}}\;
\mathop{\mathbb{E}}\nolimits_{(q,y)\sim\mathcal{D}_{\text{ev}}}\!\Bigl[\,
R_{\text{res}}\bigl(O_{\pi_\theta,\mathcal{L}}(q)\bigr)
\;-\;
\mathop{\mathbb{E}}\nolimits_{\tau}\!\bigl[\,T(\tau)\cdot\bar C_{\text{retr}}(q,\mathcal{L})\,\bigr]
\,\Bigr]\\
&\text{s.t.}\;\;
\sum\nolimits_{v\in\mathcal{S}_{\ell-1}}\!c_\ell(v)\;\le\;W
\qquad\forall\,\ell\in\{2,3,4\},
\end{aligned}
\end{equation}
% \vspace{-0.5em}
where the second term is the expected token cost of solving $q$.
Retrieval issues three LLM calls (at $L_2,L_3,L_4$), each materialising
$L_\ell(v)$ for surviving $v\in\mathcal{S}_{\ell-1}$; $L_1$ is a
symbolic type filter and incurs no LLM call. After $L_4$ each typed
sub-goal commits to a single winner that the executor invokes
directly, so the context-window limit applies \emph{per retrieval
call}, giving the three constraints above. None of these quantities is
method-specific: $R_{\text{res}}$ is the task verifier,
$T\cdot\bar C_{\text{retr}}$ is the standard token bill of any
tool-augmented LLM, and $W$ is a model context constant.
\textbf{Three structural axes.}
Eq.~\eqref{eq:objective} exposes three orthogonal cost axes, each
addressed by a different component:
\textbf{(A)} \emph{Per-call retrieval cost
$\bar C_{\text{retr}}$.} Flat retrieval gives
$\bar C_{\text{retr}}=\sum_v\sum_\ell c_\ell(v)$, linear in
$|\mathcal{V}|$. Typed DAG Retrieval (Section~\ref{sec:hier-index})
makes this sublinear via a cascade
$\mathcal{S}_1\!\supseteq\!\cdots\!\supseteq\!\mathcal{S}_4$ paying
$c_\ell$ only on survivors;
\textbf{(B)} \emph{Number of tool calls $T(\tau)$.} A depth-$d$
composite subsuming an $m$-step primitive sub-trajectory replaces
$m$ planner calls with one. Co-Evolution (Section~\ref{sec:coevo})
shrinks $T(\tau)$ by growing $\bar d(\mathcal{V}_c)$ via
teacher-distilled \textsc{InsertTool} commits;
\textbf{(C)} \emph{Per-call context limit $W$.} Each
$\ell\in\{2,3,4\}$ is a separate LLM call that fit in $W$ on its
own; a flat single call scaling with $|\mathcal{V}|$ breaches $W$ on
any non-trivial library. The cascade pays $c_\ell$ only on
$\mathcal{S}_{\ell-1}$, so each call fits regardless of
$|\mathcal{V}|$.

\textbf{Roadmap.}
\emph{Typed DAG Retrieval} (Section~\ref{sec:hier-index})
handles axes (A) and (C) via the cascade
$L_1\!\to\!L_2\!\to\!L_3\!\to\!L_4$, paying each $c_\ell$ only on
survivors so per-stage calls fit under $W$ and average retrieval cost
is sublinear in $|\mathcal{V}|$. At training, \emph{Co-Evolution under
a graph-aware reward} (Section~\ref{sec:coevo}) handles axis (B):
\textsc{InsertTool} commits grow $\bar d(\mathcal{V}_c)$ while a
structural reward credits the planner for routing through deeper
composites, shrinking $T(\tau)$. The two close the loop: a deeper
library makes retrieval cheaper per call and shorter, while
cheaper retrieval keeps the prompt under $W$ as the library grows.
% \subsection{Typed DAG Retrieval: Approximating the Retrieval Cost Block}
% \label{sec:hier-index}
\vspace{-0.5em}
%%% TODO: token cost 的关系要表达出来:
%%% 1. 定理 2. 图上表达 3. 文中要写时间复杂度
%%% 要体现出到前面就停的优势 时间复杂度 用期望
\subsection{Typed DAG Retrieval: Approximating the Retrieval-Cost Block}
\label{sec:hier-index}
\vspace{-0.5em}
This section addresses axes (A) and (C) of Eq. (1). Typed DAG Retrieval is coarse-to-fine along two orthogonal axes. The first is the \emph{abstraction axis} induced by the compositional DAG: retrieval starts from the highest-level tool whose signature can cover the current subgoal and descends to children only when that composite is too broad, incompatible, or insufficiently discriminative. Thus, the retriever stops as soon as a validated composite can solve the subgoal, instead of blindly opening all lower-level subskills. The second is the \emph{evidence axis} inside each candidate node: compatibility is judged using increasingly informative and increasingly expensive records, from typed signatures to descriptions, specifications, and examples. These two axes play different roles. The abstraction axis decides \emph{how deep} retrieval must go in the executable skill DAG, while the evidence axis decides \emph{which alternative implementation} at the current abstraction level should be surfaced.

The key asymmetry is that tool compatibility becomes increasingly expensive to judge. Type compatibility is symbolic; semantic relevance requires description ranking; behavioral compatibility requires contract checking; and fine-grained disambiguation requires examples. CoCoDA orders these checks from cheapest to most expensive and pays each expensive cost only on the survivors of earlier stages.

For a query-induced subgoal, Typed DAG Retrieval constructs a nested sequence $V = S_0 \supseteq S_1 \supseteq S_2 \supseteq S_3 \supseteq S_4$. The final set $S_4$ is surfaced to the planner or executor. Importantly, this survivor cascade is applied lazily over the DAG: a composite node is expanded into its children only if the cascade cannot certify the composite itself as a sufficient tool for the subgoal.
\vspace{-1em}
\paragraph{L1: symbolic signature pruning.}
The first stage drops any tool whose signature does not unify with the
sub-goal:
$\mathcal{S}_1=\{v\in\mathcal{V}:L_1(v)\text{ unifies with the sub-goal signature}\}$.
Unification is decided symbolically on a static type lattice via an
inverted index, so the stage incurs zero prompt-token cost; only its
symbolic work scales with $|\mathcal{V}|$, bounded by
Corollary~\ref{cor:retrieval-time}. Running it first lets every later
LLM-billed stage operate on a much smaller survivor set.
\vspace{-1em}
\paragraph{L2: edge-guided semantic expansion.}
The second stage exposes descriptions $L_2(v)$ only for
$v\in\mathcal{S}_1$; the planner ranks these by semantic relevance to
$q$ and keeps the top-$k_2$ shortlist $\mathcal{S}_2$. Because
candidates live in an invocation DAG rather than a flat pool, ranking
can climb to a subsuming composite or descend to a primitive needed to
complete the sub-goal—structural moves that similarity-only
hierarchies cannot make. Cost is paid on $\mathcal{S}_1$, not
$\mathcal{V}$.
\vspace{-1em}
\paragraph{L3: specification filtering.}
The third stage exposes pre/post-conditions $L_3(v)$ only for
$v\in\mathcal{S}_2$ and treats them as a hard accept/reject judged by
the planner under the prompt of Appendix~\ref{app:prompt-l3}: a
candidate is dropped if its preconditions are not satisfiable in the
current sub-goal or its postconditions do not imply the desired
outcome. Binary verdicts admit no false positives under faithful
specification reading, so a surviving composite is certified as a
macro-action equivalent to its primitive sub-trajectory
(Theorem~\ref{thm:compositional}). Cost is paid only on the
$k_2$-shortlist.
\vspace{-1em}
\paragraph{L4: example-based reranking.}
The final stage exposes worked examples $L_4(v)$ only for
$v\in\mathcal{S}_3$ and reranks the survivors to break near-ties among
tools with similar signatures and contracts. 

Since $L_1$ is free, the total LLM-billed retrieval cost of the cascade is
$\sum_{\ell=2}^{4}\!\sum_{v\in\mathcal{S}_{\ell-1}}\!c_\ell(v)$,
which Theorem~\ref{thm:retrieval} bounds by
$\alpha_1\alpha_2\,C_{\text{flat}}+o(C_{\text{flat}})$ and
Corollary~\ref{cor:retrieval-time} converts to sublinear time.
\vspace{-0.5em}
\paragraph{Co-evolution tightens the cascade.}
The decomposition above is independent of which nodes occupy
$\mathcal{V}$. When co-evolution (Sec.~\ref{sec:coevo}) folds a
length-$m$ primitive sub-trajectory into one validated composite,
that composite occupies a single slot of $\mathcal{S}_2$ instead of
$m$, strictly contracting every later $\mathcal{S}_\ell$ on subsequent
queries (Remark~\ref{rem:coevo-budget}). The same abstraction that
shortens planner trajectories therefore also shortens future retrieval
prompts.
\vspace{-0.5em}
\subsection{Co-Evolution as a Coupled Update on the Joint Objective}
\label{sec:coevo}
\vspace{-0.5em}
Section~\ref{sec:hier-index} treats $(\pi_\theta,\mathcal{L})$ as
fixed: any sub-trajectory the planner stitches from primitives is
consumed for the current query and not written back. Updating only
$\pi_\theta$ leaves $\bar d(\mathcal{V}_c)$ unchanged, so every
subsequent query keeps paying the same retrieval cost in
Eq.~\eqref{eq:objective}. We therefore approach the joint objective
via a \emph{coupled} update that co-evolves $\pi_\theta$ and
$\mathcal{L}$ on the same gap signal, with $\mathcal{M}_T$ fixed.
The library is edited via atomic single-tool operators: \textsc{Add}
introduces a primitive or composite with depth assigned by
Algorithm~\ref{alg:insert-tool}; \textsc{Merge} consolidates $t^\star$
with a semantically overlapping entry; \textsc{Reject} prunes
proposals failing dependency specifications or acyclicity. They are
jointly closed under composition: any acyclic $\mathcal{L}'$ is
reachable from any $\mathcal{L}$ by a finite sequence of these
operations.
\vspace{-0.5em}
\paragraph{Surrogate Reward}
\label{sec:coevo-reward}
\vspace{-0.5em}
The cost block is $\mathbb{E}_\tau[T(\tau)\cdot\bar
C_{\text{retr}}(q,\mathcal{L})]$. Within a single rollout
$\bar C_{\text{retr}}$ is fixed by $\mathcal{L}$ and retrieval, not
$\pi_\theta$, so minimising the cost block over the policy reduces to
maximising $-T(\tau)$. By the definition of $\Phi$, every invocation
satisfies $\mathrm{flat}(t_i)-\Phi(t_i)=1$, so summing over the
trajectory gives
\begin{equation}\label{eq:cost-decomp}
-T(\tau)\;=\;\sum\nolimits_{i=1}^{T(\tau)}\bigl(\Phi(t_i)-\mathrm{flat}(t_i)\bigr)
\;=\;\underbrace{\sum\nolimits_{i=1}^{T(\tau)}\Phi(t_i)}_{\text{(i) saved-call credit}}
\;-\;\underbrace{T_{\text{prim}}(\tau)}_{\text{(ii) primitive workload}},
\end{equation}
where $T_{\text{prim}}(\tau):=\sum_i\mathrm{flat}(t_i)$ counts
primitive invocations after expansion. Term (ii) depends only on the
sub-computation $\tau$ realises, not on the composite-vs-primitive
phrasing: lossless substitution leaves $T_{\text{prim}}$ invariant.
Within a GRPO group on the same $q$, $T_{\text{prim}}$ acts as a
query-conditional baseline absorbed by group-mean subtraction in the
advantage, leaving only term (i) as a non-trivial gradient signal. We
therefore take
$R_{\text{comp}}(\tau)\;=\;\sum\nolimits_{i=1}^{T(\tau)}\Phi(t_i)\;=\;\sum\nolimits_{i=1}^{T(\tau)}\bigl(\mathrm{flat}(t_i)-1\bigr)$
as the surrogate reward. The shaped training reward is
$R(\tau)=R_{\text{res}}(\tau)+\lambda\,R_{\text{comp}}(\tau)$, with
$\lambda$ rescaling saved-call credit (units of ``primitive calls'')
against the unit-interval verifier reward. Length-cheating is
structurally impossible: padding a primitive contributes $\Phi=0$, so
$R_{\text{comp}}$ rewards only actual composite expansion.

$R_{\text{comp}}$ depends on $\mathcal{G}$ only through the scalar
$\mathrm{flat}(t_i)$, not directly through edges or depth. The DAG's
structural content is used elsewhere: retrieval traverses edges to
bound per-call cost, \textsc{InsertTool} uses them for acyclicity and
child-specification validation, and Theorem~\ref{thm:monotone} relies on
DAG inclusion $\mathcal{V}_k\subseteq\mathcal{V}_{k+1}$. The reward
inherits the DAG's guarantees because $\mathrm{flat}$ is a structural
invariant of recursive expansion, not a teacher annotation; without
the DAG, $\mathrm{flat}$ would be unverifiable and $R_{\text{comp}}$
trivially exploitable.

For each $q$ we run typed-DAG retrieval and rollout to obtain
$\{\tau^{(i)}\}_{i=1}^{G}$ scored by $R$. Library writes are gated on
$R_{\text{res}}$ alone, so $R_{\text{comp}}$ cannot inject unsolved
trajectories into $\mathcal{L}$; it acts only on the planner.
\vspace{-1em}
\paragraph{Coupled Update Rule}
\label{sec:coevo-update}
\vspace{-0.5em}
At iteration $k$, given gap-signal-positive rollouts
$\mathcal{T}^{+}_q=\{\tau^{(i)}:R_{\text{res}}(\tau^{(i)})\geq\rho\}$,
we update $(\pi_\theta,\mathcal{L})$ simultaneously:
\begin{equation}
\label{eq:coupled-update}
\pi_{\theta_{k+1}}\!=\!\textsc{GRPO}\bigl(\pi_{\theta_k};\{(\tau^{(i)},R(\tau^{(i)}))\}\bigr),
\qquad
\mathcal{L}_{k+1}\!=\!\textsc{Fold}\bigl(\textsc{InsertTool},\,\mathcal{M}_T(\mathcal{T}^{+}_q),\,\mathcal{L}_k\bigr),
\end{equation}
where the planner update is a clipped GRPO step on
$R=R_{\text{res}}+\lambda R_{\text{comp}}$, and \textsc{Fold} applies
\textsc{InsertTool} to each candidate
$t^\star\in\mathcal{M}_T(\mathcal{T}^{+}_q)$ in sequence, performing
\textsc{Add}/\textsc{Merge}/\textsc{Reject} after spec and acyclicity
validation. The two updates jointly attack axis (B): each write
raises $\bar d(\mathcal{V}_c)$, letting a fixed planner solve the
query with fewer calls; each GRPO step shifts $\pi_\theta$ toward deeper composites, biasing
$\mathcal{M}_T(\mathcal{T}^{+}_q)$ toward reusable abstractions and
accelerating the next commit.

\textbf{Why both updates are needed.}
Updating only $\pi_\theta$ leaves $\bar d(\mathcal{V}_c)$ untouched
and the planner overfits to a fixed retrieval surface; updating only
$\mathcal{L}$ leaves the planner unable to invoke new composites.
Because Eq.~\eqref{eq:coupled-update} co-evolves both on the same gap
signal, and $(\pi_{\theta_k},\mathcal{L}_k)$ is always a feasible
fallback (clipped GRPO and \textsc{Reject} default to identity on
failed proposals), the sequence
$\{J(\pi_{\theta_k},\mathcal{L}_k)\}_k$ is monotone non-decreasing
under a small enough step size (Theorem~\ref{thm:monotone}).
\vspace{-0.5em}
\subsection{Theoretical Analysis}
\label{sec:theory}
\vspace{-0.5em}
Here we analyze our design as an objective-consistent approximation
to Eq.~\eqref{eq:objective}. We show that Typed DAG Retrieval makes
the cost block sublinear, $R_{\text{comp}}$ favors deeper compositions,
and the coupled update monotonically improves
$J(\pi_\theta,\mathcal{L})$, with the same DAG depth driving both
sides. Due to limited space, full statements and proofs are deferred
to Appendix~\ref{app:proofs}.

\begin{theorem}[Retrieval Cost Reduction]
\label{thm:retrieval}
Let $C_{\text{flat}}=\sum_{v\in\mathcal{V}}\sum_{\ell=1}^{4}c_\ell(v)$ be the
expected context cost of the flat baseline, and
$C_{\text{hier}}=\mathbb{E}[\sum_{\ell=1}^{4}\sum_{v\in\mathcal{S}_{\ell-1}}c_\ell(v)]$
that of \textsc{TypedDAGRetrieve} (Algorithm~\ref{alg:hier-retrieve}), where
$\mathcal{S}_\ell$ is the surviving set after stage $L_\ell$ and
$\mathcal{S}_0:=\mathcal{V}$. Under Assumption~\ref{asm:level-cost},
$C_{\text{hier}}\le \alpha_1\alpha_2\,C_{\text{flat}}+o(C_{\text{flat}})=o(C_{\text{flat}})$
as $n\!\to\!\infty$.
\end{theorem}

\begin{corollary}[Sublinear Retrieval Time]
\label{cor:retrieval-time}
Let $\tau_\ell$ be the per-candidate wall-clock cost at $L_\ell$ with
$\tau_1\ll\tau_2,\tau_3,\tau_4$. Under Assumption~\ref{asm:level-cost} and the
inverted-index L1 (Appendix~\ref{app:proofs}),
$T_{\text{hier}}=\mathcal{O}(\tau_1\log n+(\tau_2+\tau_3+\tau_4)k_2)$, sublinear
in $n=|\mathcal{V}|$, vs.\ $T_{\text{flat}}=\Theta((\tau_2+\tau_3+\tau_4)n)$.
\end{corollary}

\begin{remark}[Co-Evolution Tightens the Prompt Budget]
\label{rem:coevo-budget}
Since Corollary~\ref{cor:retrieval-time} is independent of which nodes
occupy the DAG, every commit by Algorithm~\ref{alg:insert-tool} replaces
a length-$m$ primitive sub-trajectory with a single composite token
(Theorem~\ref{thm:compositional}), contracting the L2/L3/L4 candidate
set of Algorithm~\ref{alg:hier-retrieve}. Co-evolution and
context-bounded retrieval are thus one mechanism: deeper
$\bar d(\mathcal{V}_c)$ both credits the policy under graph-aware GRPO
and shortens its prompt.
\end{remark}

\begin{theorem}[Compositional Advantage under Graph-Aware GRPO]
\label{thm:compositional}
Let $\tau_p$ invoke only primitives, and let $\tau_c$ replace a contiguous
length-$m$ ($m\ge 2$) sub-trajectory of $\tau_p$ by a composite
$t^\star\in\mathcal{V}_c$ with $\mathrm{flat}(t^\star)\ge m$. If
$R_{\text{res}}(\tau_p)=R_{\text{res}}(\tau_c)$ and $\lambda>0$, then
$R(\tau_c)-R(\tau_p)=\lambda\,\Phi(t^\star)\ge\lambda(m-1)>0$, so $\tau_c$
receives strictly larger group-relative advantage under GRPO.
\end{theorem}

\begin{theorem}[Monotone Co-Evolution]
\label{thm:monotone}
Let $(\pi_{\theta_k},\mathcal{L}_k)$ be the planner--library pair after the
$k$-th update in Stage 3 of Algorithm~\ref{alg:todo-main}, and
$J(\pi,\mathcal{L})=\mathbb{E}_{q,\tau}[R_{\text{res}}+\lambda R_{\text{comp}}]$.
Under Assumptions~\ref{asm:verifier}--\ref{asm:retrieval}, there exists
$\bar\eta>0$ such that for every GRPO learning rate $\eta\in(0,\bar\eta]$,
$J(\pi_{\theta_k},\mathcal{L}_k)$ is monotonically non-decreasing in $k$.
\end{theorem}

\begin{corollary}[DAG Well-Formedness and Complexity]
\label{cor:dag}
For all $k\ge 0$, $\mathcal{G}_k$ from Algorithm~\ref{alg:insert-tool} is
acyclic with $d(v)=1+\max_{u\in\mathrm{Ch}(v)}d(u)$ for every
$v\in\mathcal{V}_c$; max depth is $\mathcal{O}(\log|\mathcal{V}|)$ when
fan-out $\ge 2$. \textsc{InsertTool} costs
$\mathcal{O}(|\mathrm{Ch}(t^\star)|+|\mathcal{E}|)$ per insertion plus
$\mathcal{O}(1)$ amortized indexing.
\end{corollary}
\vspace{-1em}
\section{Experiments}
\label{sec:experiments}
\vspace{-0.5em}
We evaluate \textit{CoCoDA} on a cloud VM with $4$ NVIDIA H200 GPUs. Qwen3-32B serves as the teacher for tool extraction and solution generation; smaller students are trained via GRPO. Hyperparameters and reward weights are in the appendix. Source code: \url{https://anonymous.4open.science/r/CoCoDA-664C}.
\subsection{Experimental Settings}
\vspace{-0.5em}
\textbf{Datasets.} We use six benchmarks across three categories with deterministic code-based verification: math/logical reasoning (GSM8K~\citep{cobbe2021gsm8k}, MATH~\citep{hendrycks2021math}), tabular analysis (WikiTableQuestions~\citep{pasupat2015wikitable}, FinQA~\citep{chen2021finqa}), and code task (EvalPlus~\citep{liu2023evalplus}, MBPP~\citep{austin2021mbpp}). Statistics are in the appendix.

\textbf{Comparison Methods.} We compare against two no-tool CoT baselines (Qwen3 student and Qwen3-32B teacher~\citep{wei2023chain}), three fixed-library tool-augmented methods (ReAct~\citep{react2023}, ToRL~\citep{torl2024}, ReTool~\citep{retool2025}), and two library-learning methods (CREATOR~\citep{creator2023}, TroVE~\citep{trove2024}). All baselines share the student backbone and task prompts. Component-level ablations are in Section~\ref{sec:ablation}.
\vspace{-0.5em}
\subsection{Main Results.}
\vspace{-0.5em}
Table~\ref{tab:main_results} reports accuracy on six benchmarks across four Qwen3 student sizes ($0.6$B--$8$B), with the Qwen3-32B teacher as an upper bound. \textit{CoCoDA} wins every (size, benchmark) cell and beats the strongest tool-augmented baseline, ReTool, by up to $2$ points. Gains over vanilla CoT are largest at the smallest scale ($+10.7$ on GSM8K, $+10.8$ on MBPP at 0.6B), showing that the co-evolved library transfers teacher competence most effectively when student capacity is scarce. The teacher gap closes with scale: at 8B, \textit{CoCoDA} \emph{matches or exceeds} the 32B teacher on math/logic (GSM8K $93.67$ vs.\ $93.40$, MATH $63.18$ vs.\ $61.62$;~$\star$ in Table~\ref{tab:main_results}), recovering teacher-level performance at roughly a quarter of the parameters. Comparable margins on symbolic (GSM8K, MATH) and retrieval-heavy (WTQ, FinQA) tasks indicate the library generalizes across reasoning styles rather than overfitting one skill. In contrast, library-construction baselines without a tight student--teacher loop (CREATOR, TroVE) \emph{underperform} plain CoT on math and code, confirming tool creation helps only when the library co-evolves with the learner.
\vspace{-1em}
\begin{table}[htbp]
\centering
\small
\setlength{\tabcolsep}{5pt}
\renewcommand{\arraystretch}{1.1}
\vspace{-1.5em}
\caption{Main results across three task categories and six benchmarks, reported for 
four student sizes of the Qwen3 family. Within each student-size block, \textbf{bold} denotes the best result 
and \underline{underline} the second-best. CoCoDA entries that match or exceed the teacher are additionally marked 
with~$\star$.}
\label{tab:main_results}
\begin{tabular}{l l cc cc cc}
\toprule
 & & \multicolumn{2}{c}{\textbf{Math / Logical}} 
 & \multicolumn{2}{c}{\textbf{Tabular Analysis}} 
 & \multicolumn{2}{c}{\textbf{Code Task}} \\
\cmidrule(lr){3-4} \cmidrule(lr){5-6} \cmidrule(lr){7-8}
\textbf{Size} & \textbf{Method}
 & GSM8K & MATH 
 & WTQ   & FinQA 
 & EP & MBPP \\
\midrule
32B & Teacher (CoT)        & 93.40 &  61.62 & 51.9 & 48.5 & 82.3 & 78.20 \\
\midrule
\multirow{7}{*}{0.6B}
    & Student (CoT)                  & 59.59 &  32.44 & 17.98 & 5.93 & 19.5 & 36.60 \\
    & ReAct~\citep{react2023}         & 52.13 & 28.18 & 15.57 & 7.65 & 17.60 & 32.11 \\
    & ToRL~\citep{torl2024}           & 65.75 & 36.24 & 22.59 & 8.55 & 24.06 & 42.92 \\
    & ReTool~\citep{retool2025}       & 68.85 & 39.65 & 24.22 & 10.55 & 27.78 & 46.30 \\
    & CREATOR~\citep{creator2023}     & 48.60 & 20.93 & 12.77 & 7.89 & 15.33 & 28.92 \\
    & TroVE~\citep{trove2024}         & 50.74 & 23.85 & 13.86 & 8.49 & 16.62 & 30.16 \\
    \rowcolor{gray!20}
    & \textbf{CoCoDA (Ours)}           & \textbf{70.27} & \textbf{40.13} & \textbf{25.68} & \textbf{11.24} & \textbf{28.47} & \textbf{47.38} \\
\midrule
\multirow{7}{*}{1.7B}
    & Student (CoT)                  & 75.44 & 43.50 & 30.32 & 16.04 & 39.6 & 55.40  \\
    & ReAct~\citep{react2023}         & 68.28 & 38.75 & 26.77 & 16.47 & 35.44 & 50.69 \\
    & ToRL~\citep{torl2024}           & 79.99 & 47.41 & 34.61 & 18.97 & 43.30 & 59.63 \\
    & ReTool~\citep{retool2025}       & 82.33 & 50.49 & 37.59 & 20.09 & 46.80 & 62.22 \\
    & CREATOR~\citep{creator2023}     & 60.88 & 25.16 & 18.26 & 16.72 & 30.76 & 42.01 \\
    & TroVE~\citep{trove2024}         & 62.27 & 29.66 & 20.17 & 18.66 & 32.51 & 45.64 \\
    \rowcolor{gray!20}
    & \textbf{CoCoDA (Ours)}           & \textbf{84.36} & \textbf{51.62} & \textbf{38.05} & \textbf{21.37} & \textbf{47.29} & \textbf{63.85} \\
\midrule
\multirow{7}{*}{4B}
    & Student (CoT)                  & 87.79 & 54.10 & 41.67 & 24.93 & 60.4 & 67.00  \\
    & ReAct~\citep{react2023}         & 82.71 & 49.48 & 36.80 & 25.82 & 56.84 & 62.95 \\
    & ToRL~\citep{torl2024}           & 89.16 & 56.65 & 44.77 & 26.17 & 62.25 & 69.68 \\
    & ReTool~\citep{retool2025}       & 91.05 & 58.41 & 46.72 & 28.95 & 64.82 & 72.27 \\
    & CREATOR~\citep{creator2023}     & 70.35 & 30.43 & 25.86 & 25.94 & 54.98 & 54.84 \\
    & TroVE~\citep{trove2024}         & 72.36 & 34.16 & 27.62 & 26.88 & 56.26 & 56.12 \\
    \rowcolor{gray!20}
    & \textbf{CoCoDA (Ours)}           & \textbf{92.64} & \textbf{59.37} & \textbf{46.83} & \textbf{29.06} & \textbf{66.28} & \textbf{73.42} \\
\midrule
\multirow{7}{*}{8B}
    & Student (CoT)                  & 89.84 & 60.80 & 47.38 & 37.31 & 70.1 & 69.80 \\
    & ReAct~\citep{react2023}         & 86.84 & 55.47 & 42.64 & 34.69 & 63.52 & 66.54 \\
    & ToRL~\citep{torl2024}           & 91.83 & 61.26 & 48.89 & 38.19 & 72.19 & 72.88 \\
    & ReTool~\citep{retool2025}       & 92.82 & 61.84 & 50.36 & 39.65 & 74.35 & 75.38 \\
    & CREATOR~\citep{creator2023}     & 80.86 & 36.18 & 30.42 & 27.29 & 58.40 & 58.42 \\
    & TroVE~\citep{trove2024}         & 82.40 & 40.57 & 32.20 & 28.90 & 60.10 & 60.77 \\
    \rowcolor{gray!20}
    & \textbf{CoCoDA (Ours)}           & \textbf{93.67}$^{\star}$ & \textbf{63.18}$^{\star}$ & \textbf{51.54} & \textbf{40.16} & \textbf{75.26} & \textbf{76.39} \\
\bottomrule
\end{tabular}
\vspace{-1.5em}
\end{table}

% \vspace{-0.5em}
\subsection{Scalability and Efficiency Analysis}
\vspace{-0.5em}
\paragraph{Library Growth.}
Figure~\ref{fig:library_growth} tracks library size and mean compositional depth of 
invoked tools over GRPO training. Both rise sharply in the first $\sim$150 steps and 
saturate, showing that \textit{CoCoDA} 
commits to a compact set of reusable abstractions rather than growing unboundedly. 
Asymptotic size varies with task complexity (larger for symbolic math than code), 
while depth stabilizes near $3$--$4$, consistent with Corollary~\ref{cor:dag}.
\vspace{-1em}
\paragraph{Context Cost vs.\ Library Size: Code Hierarchy vs.\ Text Hierarchy.}
TDR's novelty lies not in hierarchy itself but in exploiting properties unique to executable code.
Fixing the library at the full-\textit{CoCoDA} snapshot, we swap only the retrieval substrate and
measure mean prompt-token cost for a top-$k$ set as library size sweeps from $50$ to $1{,}600$ nodes
(Figure~\ref{fig:retrieval_cost}), at matched 4B-student accuracy. \textbf{Flat retrieval} embeds and
ranks all four annotation levels per node (no hierarchy; typical prior baseline). \textbf{Text-hierarchical
RAG} builds a RAPTOR-style summary tree over L2 descriptions (generic hierarchy). \textbf{Typed DAG
Retrieval (ours)} traverses the call-induced DAG (Algorithm~\ref{alg:hier-retrieve}). At $1{,}600$ nodes,
flat retrieval costs $\approx\!11.4\times$ TDR's tokens and text-hierarchical RAG $\approx\!4.7\times$,
since its topic-clustered tree cannot prune by typed signature or contract on abstracted sub-trajectories.
This gap is the empirical content of the three code-specificity claims (typed prune, edge-guided
expansion, behavior-preserving leaf substitution) of Section~\ref{sec:hier-index}.
\begin{figure}[htbp]
\centering
\begin{subfigure}[t]{0.6\linewidth}
\centering
\includegraphics[width=\linewidth]{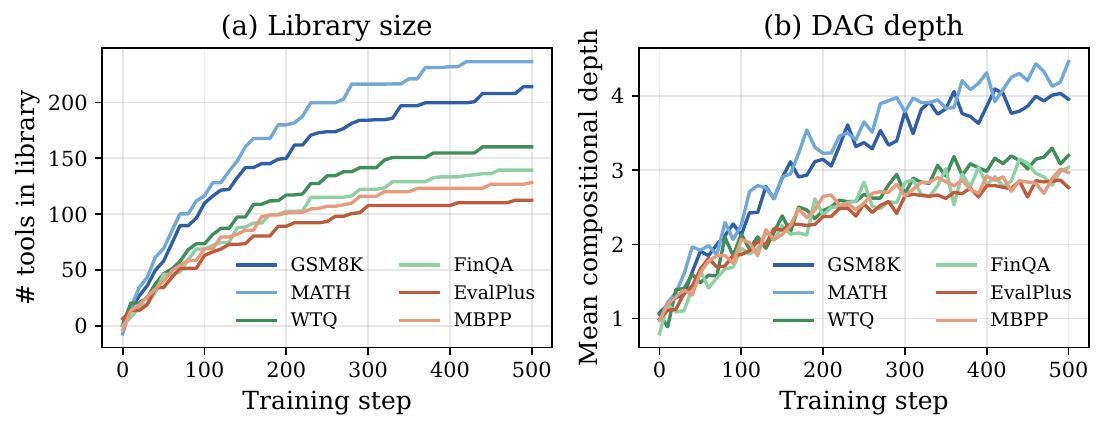}
\caption{Library co-evolution dynamics}
\label{fig:library_growth}
\end{subfigure}
\hfill
\begin{subfigure}[t]{0.35\linewidth}
\centering
\includegraphics[width=\linewidth]{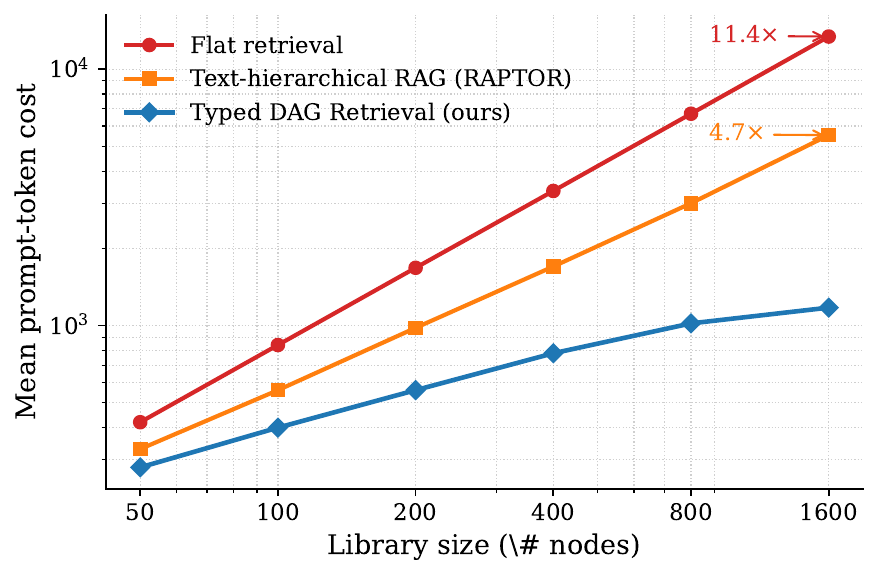}
\caption{Retrieval context-cost scaling}
\label{fig:retrieval_cost}
\end{subfigure}
\caption{\textbf{Left.} Co-evolution dynamics of the tool library under GRPO training. \textbf{Right.} 
Mean prompt-token cost to surface a top-$k$ tool set as library size grows.}
\label{fig:library_and_retrieval}
\end{figure}
\vspace{-1.1em}
\paragraph{Scaling Analysis.}
Figure~\ref{fig:scaling} sweeps library size with the 4B student fixed (one panel per benchmark). Accuracy rises steeply over the first $\sim$200 tools and plateaus by $\sim$400, while latency grows roughly linearly, so $400$ tools dominate $800$ on both accuracy and cost. Sweeping the student with the full library fixed yields monotone GSM8K gains from $0.6$B to $8$B, but the $4$B$\to$$8$B gain is only $+1.0$ while latency nearly doubles—$4$B is the accuracy--cost knee.
{\setlength{\intextsep}{2pt}%
\setlength{\floatsep}{2pt}%
\setlength{\textfloatsep}{2pt}%
\begin{figure}[!htbp]
\centering
\vspace{-1em}
\includegraphics[width=0.7\linewidth]{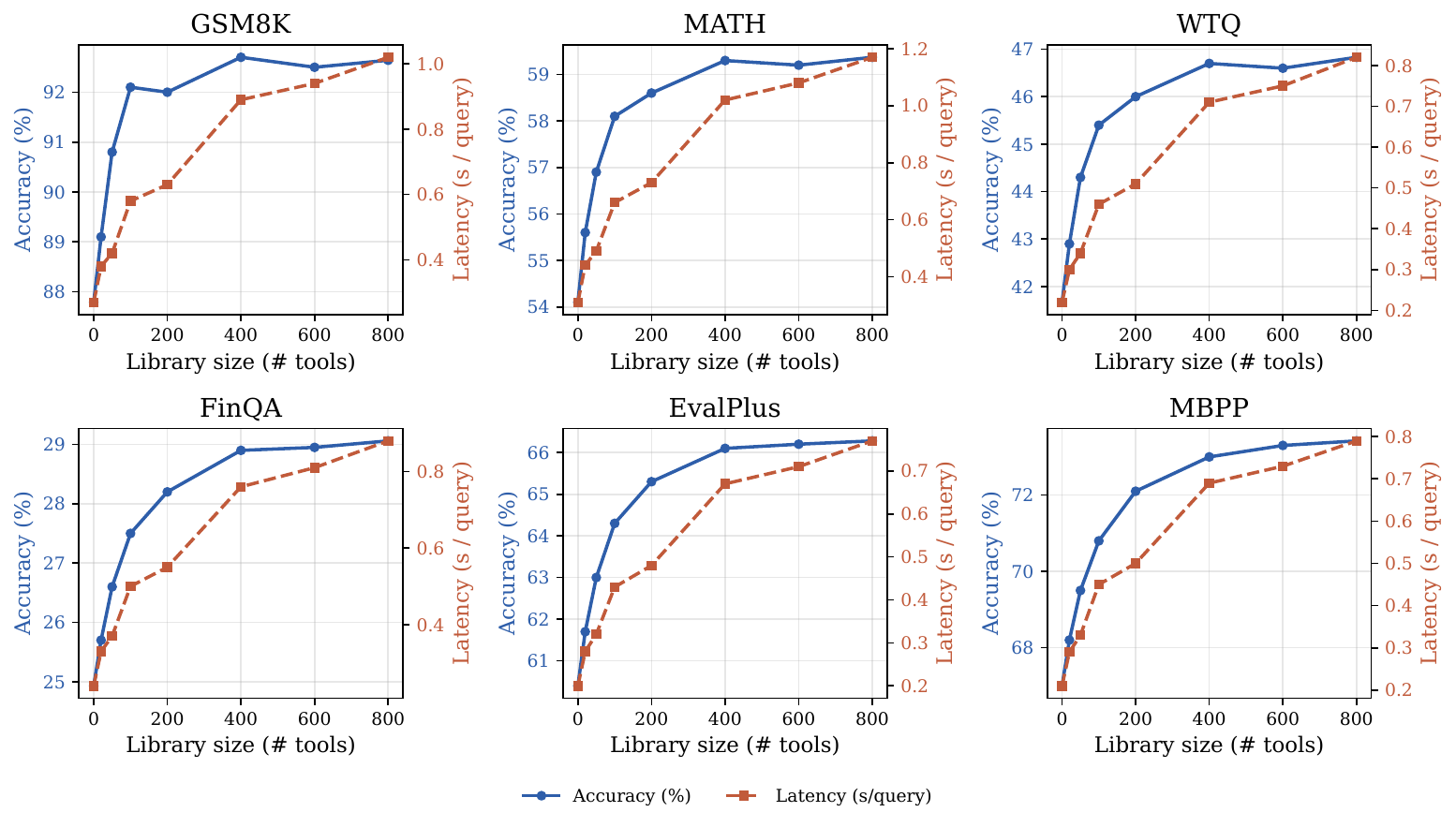}
\vspace{-0.8em}
\caption{Accuracy and per-query latency as a function of library size at a
fixed $4$B student.}
\label{fig:scaling}
\vspace{-1.5em}
\end{figure}}
\vspace*{-1em}
\subsection{Ablation Studies}
\label{sec:ablation}
\vspace{-0.5em}
We ablate \textit{CoCoDA}'s three components: the \emph{Compositional Tool DAG} (CTD,
reusable tools with dependency edges); \emph{Typed DAG Retrieval} (TDR, typed pruning
with edge-guided expansion and example disambiguation); and the \emph{graph-aware reward}
(GAR, a structural credit term for tools on verified DAG paths). Each variant retrains
the 4B student under identical hyperparameters, replacing only the ablated component with
its non-structural counterpart (flat library, dense retrieval, plain reward); see
Table~\ref{tab:ablation}. Removing CTD causes the largest drop ($2$--$4$ points everywhere):
without dependency edges the planner re-derives intermediates and the library collapses
to a CREATOR/TroVE-style flat collection. Disabling TDR yields a milder drop concentrated
on tabular tasks (WTQ, FinQA) where the candidate set is largest; this reduces to flat
dense retrieval—the configuration to which text-style hierarchical RAG and similarity-clustered
skill memories collapse once code-specific structure is removed, so the $-1.79$ mean loss
isolates what \emph{code-specific} hierarchy buys beyond any text/skill-hierarchy substrate.
Removing GAR weakens credit assignment, costing $1$--$2$ points uniformly, most visibly on
MATH and MBPP (longest tool chains). Disabling all three barely beats plain CoT, confirming
the components are complementary.

\vspace{-1em}
\begin{table}[htbp]
\centering
\small
\setlength{\tabcolsep}{5pt}
\renewcommand{\arraystretch}{1.1}
% \vspace{-2em}
\setlength{\abovecaptionskip}{2pt}
\setlength{\belowcaptionskip}{2pt}
\caption{Ablation of \textit{CoCoDA}'s three core components on the 4B Qwen3 student.}
\label{tab:ablation}
\begin{tabular}{l cc cc cc c}
\toprule
 & \multicolumn{2}{c}{\textbf{Math / Logical}}
 & \multicolumn{2}{c}{\textbf{Tabular Analysis}}
 & \multicolumn{2}{c}{\textbf{Code Task}}
 & \\
\cmidrule(lr){2-3} \cmidrule(lr){4-5} \cmidrule(lr){6-7}
\textbf{Variant}
 & GSM8K & MATH
 & WTQ   & FinQA
 & EP & MBPP
 & $\Delta$ \\
\midrule
\rowcolor{gray!20}
\textbf{CoCoDA (full)}            & \textbf{92.64} & \textbf{59.37} & \textbf{46.83} & \textbf{29.06} & \textbf{66.28} & \textbf{73.42} & ---    \\
\midrule
w/o CTD (flat library)          & 89.12 & 55.74 & 43.51 & 26.18 & 62.83 & 69.47 & $-3.49$ \\
w/o TDR (flat dense retrieval)  & 91.08 & 57.92 & 43.86 & 26.61 & 64.73 & 71.58 & $-1.79$ \\
w/o GAR (exec-only reward)      & 91.32 & 57.48 & 45.27 & 27.74 & 65.16 & 71.82 & $-1.40$ \\
\midrule
w/o CTD + TDR                   & 88.47 & 54.63 & 41.92 & 25.38 & 61.94 & 68.52 & $-4.64$ \\
w/o CTD + TDR + GAR             & 87.96 & 54.27 & 41.38 & 24.83 & 61.27 & 67.86 & $-5.18$ \\
\bottomrule
\end{tabular}
\vspace{-1.5em}
\end{table}
% \vspace{-0.em}
\section{Conclusion}
\vspace{-0.5em}
This work addresses a key limitation of small language models equipped with external code libraries: the library must co-evolve with the policy without exhausting the planner's context budget. We presented \textit{CoCoDA}, which resolves both faces through a single \emph{compositional code DAG} that supports graph-aware GRPO co-evolution and Typed DAG Retrieval, exploiting properties only executable code admits. Experiments on GSM8K and MATH show that an 8B student matches or exceeds its 32B teacher, with the largest gains at the smallest student sizes. Future work will extend CoCoDA to multimodal tool-use and multi-agent settings, while improving composite compression for large evolving libraries.

% \section*{References}

\bibliographystyle{plainnat}  % or unsrtnat, abbrvnat
\bibliography{references}      % no .bib extension

@inproceedings{toolformer2023,
  title     = {Toolformer: Language Models Can Teach Themselves to Use Tools},
  author    = {Schick, Timo and Dwivedi-Yu, Jane and Dess{\`i}, Roberto and Raileanu, Roberta and Lomeli, Maria and Hambro, Eric and Zettlemoyer, Luke and Cancedda, Nicola and Scialom, Thomas},
  booktitle = {Advances in Neural Information Processing Systems (NeurIPS)},
  year      = {2023}
}

@inproceedings{react2023,
  title     = {{ReAct}: Synergizing Reasoning and Acting in Language Models},
  author    = {Yao, Shunyu and Zhao, Jeffrey and Yu, Dian and Du, Nan and Shafran, Izhak and Narasimhan, Karthik and Cao, Yuan},
  booktitle = {International Conference on Learning Representations (ICLR)},
  year      = {2023}
}

@inproceedings{toolllm2024,
  title     = {{ToolLLM}: Facilitating Large Language Models to Master 16000+ Real-World {APIs}},
  author    = {Qin, Yujia and Liang, Shihao and Ye, Yining and Zhu, Kunlun and Yan, Lan and Lu, Yaxi and Lin, Yankai and Cong, Xin and Tang, Xiangru and Qian, Bill and Zhao, Sihan and Hong, Lauren and Tian, Runchu and Xie, Ruobing and Zhou, Jie and Gerstein, Mark and Li, Dahai and Liu, Zhiyuan and Sun, Maosong},
  booktitle = {International Conference on Learning Representations (ICLR)},
  year      = {2024}
}

@inproceedings{gorilla2024,
  title     = {Gorilla: Large Language Model Connected with Massive {APIs}},
  author    = {Patil, Shishir G. and Zhang, Tianjun and Wang, Xin and Gonzalez, Joseph E.},
  booktitle = {Advances in Neural Information Processing Systems (NeurIPS)},
  year      = {2024}
}

@inproceedings{anytool2024,
  title     = {{AnyTool}: Self-Reflective, Hierarchical Agents for Large-Scale {API} Calls},
  author    = {Du, Yu and Wei, Fangyun and Zhang, Hongyang},
  booktitle = {International Conference on Machine Learning (ICML)},
  year      = {2024}
}

@article{toolalpaca2023,
  title   = {{ToolAlpaca}: Generalized Tool Learning for Language Models with 3000 Simulated Cases},
  author  = {Tang, Qiaoyu and Deng, Ziliang and Lin, Hongyu and Han, Xianpei and Liang, Qiao and Cao, Boxi and Sun, Le},
  journal = {arXiv preprint arXiv:2306.05301},
  year    = {2023}
}

@article{rltf2024,
  title   = {{RLTF}: Reinforcement Learning from Unit Test Feedback},
  author  = {Liu, Jiate and Zhu, Yiqin and Xiao, Kaiwen and Fu, Qiang and Han, Xiao and Yang, Wei and Ye, Deheng},
  journal = {Transactions on Machine Learning Research (TMLR)},
  year    = {2024}
}

@article{torl2024,
  title   = {{ToRL}: Scaling Tool-Integrated RL},
  author  = {Li, Xuefeng and Zou, Haoyang and Liu, Pengfei},
  journal = {arXiv preprint arXiv:2503.23383},
  year    = {2025}
}

@article{voyager2023,
  title   = {Voyager: An Open-Ended Embodied Agent with Large Language Models},
  author  = {Wang, Guanzhi and Xie, Yuqi and Jiang, Yunfan and Mandlekar, Ajay and Xiao, Chaowei and Zhu, Yuke and Fan, Linxi and Anandkumar, Anima},
  journal = {Transactions on Machine Learning Research (TMLR)},
  year    = {2024}
}

@inproceedings{latm2023,
  title     = {Large Language Models as Tool Makers},
  author    = {Cai, Tianle and Wang, Xuezhi and Ma, Tengyu and Chen, Xinyun and Zhou, Denny},
  booktitle = {International Conference on Learning Representations (ICLR)},
  year      = {2024}
}

@inproceedings{creator2023,
  title     = {{CREATOR}: Tool Creation for Disentangling Abstract and Concrete Reasoning of Large Language Models},
  author    = {Qian, Cheng and Han, Chi and Fung, Yi R. and Qin, Yujia and Liu, Zhiyuan and Ji, Heng},
  booktitle = {Findings of the Association for Computational Linguistics: EMNLP},
  year      = {2023}
}

@inproceedings{trove2024,
  title     = {{TroVE}: Inducing Verifiable and Efficient Toolboxes for Solving Programmatic Tasks},
  author    = {Wang, Zhiruo and Fried, Daniel and Neubig, Graham},
  booktitle = {International Conference on Machine Learning (ICML)},
  year      = {2024}
}

@inproceedings{regal2024,
  title     = {{ReGAL}: Refactoring Programs to Discover Generalizable Abstractions},
  author    = {Stengel-Eskin, Elias and Prasad, Archiki and Bansal, Mohit},
  booktitle = {International Conference on Machine Learning (ICML)},
  year      = {2024}
}

@inproceedings{toolrl2025,
  title     = {{ToolRL}: Reward is All Tool Learning Needs},
  author    = {Qian, Cheng and Acikgoz, Emre Can and He, Qi and Wang, Hongru and Chen, Xiusi and Hakkani-T{\"u}r, Dilek and Tur, Gokhan and Ji, Heng},
  booktitle = {Advances in Neural Information Processing Systems (NeurIPS)},
  year      = {2025}
}

@article{retool2025,
  title   = {{ReTool}: Reinforcement Learning for Strategic Tool Use in {LLMs}},
  author  = {Feng, Jiazhan and Huang, Shijue and Qu, Xingwei and Zhang, Ge and Qin, Yujia and Zhong, Baoquan and Jiang, Chengquan and Chi, Jinxin and Zhong, Wanjun},
  journal = {arXiv preprint arXiv:2504.11536},
  year    = {2025}
}

@article{artist2025,
  title   = {Agentic Reasoning and Tool Integration for {LLMs} via Reinforcement Learning},
  author  = {Singh, Joykirat and Magazine, Raghav and Pandya, Yash and Nambi, Akshay},
  journal = {arXiv preprint arXiv:2505.01441},
  year    = {2025}
}

@article{deepseekmath2024,
  title   = {{DeepSeekMath}: Pushing the Limits of Mathematical Reasoning in Open Language Models},
  author  = {Shao, Zhihong and Wang, Peiyi and Zhu, Qihao and Xu, Runxin and Song, Junxiao and Bi, Xiao and Zhang, Haowei and Zhang, Mingchuan and Li, Y. K. and Wu, Yang and Guo, Daya},
  journal = {arXiv preprint arXiv:2402.03300},
  year    = {2024}
}

@article{deepseekr1,
  title   = {{DeepSeek-R1}: Incentivizing Reasoning Capability in {LLMs} via Reinforcement Learning},
  author  = {{DeepSeek-AI} and Guo, Daya and Yang, Dejian and Zhang, Haowei and Song, Junxiao and Zhang, Ruoyu and Xu, Runxin and Zhu, Qihao and Ma, Shirong and Wang, Peiyi and others},
  journal = {arXiv preprint arXiv:2501.12948},
  year    = {2025}
}

@misc{yuan2024craft,
      title={CRAFT: Customizing LLMs by Creating and Retrieving from Specialized Toolsets}, 
      author={Lifan Yuan and Yangyi Chen and Xingyao Wang and Yi R. Fung and Hao Peng and Heng Ji},
      year={2024},
      eprint={2309.17428},
      archivePrefix={arXiv},
      primaryClass={cs.CL},
      url={https://arxiv.org/abs/2309.17428}, 
}

@article{cobbe2021gsm8k,
  title     = {Training Verifiers to Solve Math Word Problems},
  author    = {Cobbe, Karl and Kosaraju, Vineet and Bavarian, Mohammad and
               Chen, Mark and Jun, Heewoo and Kaiser, Lukasz and
               Plappert, Matthias and Tworek, Jerry and Hilton, Jacob and
               Nakano, Reiichiro and Hesse, Christopher and Schulman, John},
  journal   = {arXiv preprint arXiv:2110.14168},
  year      = {2021}
}

@inproceedings{hendrycks2021math,
  title     = {Measuring Mathematical Problem Solving With the {MATH} Dataset},
  author    = {Hendrycks, Dan and Burns, Collin and Kadavath, Saurav and
               Arora, Akul and Basart, Steven and Tang, Eric and
               Song, Dawn and Steinhardt, Jacob},
  booktitle = {Proceedings of the Neural Information Processing Systems Track on Datasets and Benchmarks},
  year      = {2021}
}

@inproceedings{pasupat2015wikitable,
  title     = {Compositional Semantic Parsing on Semi-Structured Tables},
  author    = {Pasupat, Panupong and Liang, Percy},
  booktitle = {Proceedings of the 53rd Annual Meeting of the Association for
               Computational Linguistics and the 7th International Joint
               Conference on Natural Language Processing (Volume 1: Long Papers)},
  pages     = {1470--1480},
  year      = {2015},
  address   = {Beijing, China},
  publisher = {Association for Computational Linguistics}
}

@inproceedings{chen2021finqa,
  title     = {{F}in{QA}: A Dataset of Numerical Reasoning over Financial Data},
  author    = {Chen, Zhiyu and Chen, Wenhu and Smiley, Charese and Shah, Sameena and
               Borova, Iana and Langdon, Dylan and Moussa, Reema and Beane, Matt and
               Huang, Ting-Hao and Routledge, Bryan and Wang, William Yang},
  booktitle = {Proceedings of the 2021 Conference on Empirical Methods in Natural
               Language Processing},
  pages     = {3697--3711},
  year      = {2021},
  publisher = {Association for Computational Linguistics}
}

@inproceedings{wei2023chain,
  title     = {Chain-of-Thought Prompting Elicits Reasoning in Large Language Models},
  author    = {Wei, Jason and Wang, Xuezhi and Schuurmans, Dale and Bosma, Maarten and Ichter, Brian and Xia, Fei and Chi, Ed and Le, Quoc and Zhou, Denny},
  booktitle = {Advances in Neural Information Processing Systems (NeurIPS)},
  year      = {2022}
}

@inproceedings{liu2023evalplus,
  title     = {Is Your Code Generated by {ChatGPT} Really Correct? Rigorous Evaluation of Large Language Models for Code Generation},
  author    = {Liu, Jiawei and Xia, Chunqiu Steven and Wang, Yuyao and Zhang, Lingming},
  booktitle = {Advances in Neural Information Processing Systems (NeurIPS)},
  year      = {2023}
}

@article{austin2021mbpp,
  title={Program Synthesis with Large Language Models},
  author={Austin, Jacob and Odena, Augustus and Nye, Maxwell and Bosma, Maarten and
          Michalewski, Henryk and Dohan, David and Jiang, Ellen and Cai, Carrie and
          Terry, Michael and Le, Quoc and Sutton, Charles},
  journal={arXiv preprint arXiv:2108.07732},
  year={2021}
}

@inproceedings{park2023generative,
  title     = {Generative Agents: Interactive Simulacra of Human Behavior},
  author    = {Park, Joon Sung and O'Brien, Joseph C. and Cai, Carrie J. and Morris, Meredith Ringel and Liang, Percy and Bernstein, Michael S.},
  booktitle = {Proceedings of the 36th Annual ACM Symposium on User Interface Software and Technology (UIST)},
  year      = {2023}
}

@inproceedings{packer2024memgpt,
  title     = {{MemGPT}: Towards {LLMs} as Operating Systems},
  author    = {Packer, Charles and Wooders, Sarah and Lin, Kevin and Fang, Vivian and Patil, Shishir G. and Stoica, Ion and Gonzalez, Joseph E.},
  booktitle = {Conference on Language Modeling (COLM)},
  year      = {2024}
}

@inproceedings{shinn2023reflexion,
  title     = {{Reflexion}: Language Agents with Verbal Reinforcement Learning},
  author    = {Shinn, Noah and Cassano, Federico and Gopinath, Ashwin and Narasimhan, Karthik and Yao, Shunyu},
  booktitle = {Advances in Neural Information Processing Systems (NeurIPS)},
  year      = {2023}
}

@inproceedings{madaan2023selfrefine,
  title     = {Self-Refine: Iterative Refinement with Self-Feedback},
  author    = {Madaan, Aman and Tandon, Niket and Gupta, Prakhar and Hallinan, Skyler and Gao, Luyu and Wiegreffe, Sarah and Alon, Uri and Dziri, Nouha and Prabhumoye, Shrimai and Yang, Yiming and Gupta, Shashank and Majumder, Bodhisattwa Prasad and Hermann, Katherine and Welleck, Sean and Yazdanbakhsh, Amir and Clark, Peter},
  booktitle = {Advances in Neural Information Processing Systems (NeurIPS)},
  year      = {2023}
}

@article{zhu2023ghost,
  title   = {Ghost in the {M}inecraft: Generally Capable Agents for Open-World Environments via Large Language Models with Text-based Knowledge and Memory},
  author  = {Zhu, Xizhou and Chen, Yuntao and Tian, Hao and Tao, Chenxin and Su, Weijie and Yang, Chenyu and Huang, Gao and Li, Bin and Lu, Lewei and Wang, Xiaogang and Qiao, Yu and Zhang, Zhaoxiang and Dai, Jifeng},
  journal = {arXiv preprint arXiv:2305.17144},
  year    = {2023}
}

@inproceedings{qian2024experiential,
  title     = {Experiential Co-Learning of Software-Developing Agents},
  author    = {Qian, Chen and Dang, Yufan and Li, Jiahao and Liu, Wei and Xie, Zihao and Wang, Yifei and Chen, Weize and Yang, Cheng and Cong, Xin and Liu, Zhiyuan and Sun, Maosong},
  booktitle = {Proceedings of the 62nd Annual Meeting of the Association for Computational Linguistics (ACL)},
  year      = {2024}
}

@inproceedings{dreamcoder2021,
  title     = {{DreamCoder}: Bootstrapping Inductive Program Synthesis with Wake-Sleep Library Learning},
  author    = {Ellis, Kevin and Wong, Catherine and Nye, Maxwell and Sabl{\'e}-Meyer, Mathias and Morales, Luc and Hewitt, Luke and Cary, Luke and Solar-Lezama, Armando and Tenenbaum, Joshua B.},
  booktitle = {Proceedings of the 42nd ACM SIGPLAN International Conference on Programming Language Design and Implementation (PLDI)},
  year      = {2021}
}

@inproceedings{stitch2023,
  title     = {Top-Down Synthesis for Library Learning},
  author    = {Bowers, Matthew and Olausson, Theo X. and Wong, Lionel and Grand, Gabriel and Tenenbaum, Joshua B. and Ellis, Kevin and Solar-Lezama, Armando},
  booktitle = {Proceedings of the ACM on Programming Languages (POPL)},
  year      = {2023}
}

@inproceedings{coderag2024,
  title     = {{CodeRAG-Bench}: Can Retrieval Augment Code Generation?},
  author    = {Wang, Zora Zhiruo and Asai, Akari and Yu, Xinyan Velocity and Xu, Frank F. and Xie, Yiqing and Neubig, Graham and Fried, Daniel},
  booktitle = {Findings of the Association for Computational Linguistics (NAACL)},
  year      = {2024}
}

@inproceedings{codet5plus2023,
  title     = {{CodeT5+}: Open Code Large Language Models for Code Understanding and Generation},
  author    = {Wang, Yue and Le, Hung and Gotmare, Akhilesh Deepak and Bui, Nghi D. Q. and Li, Junnan and Hoi, Steven C. H.},
  booktitle = {Proceedings of the 2023 Conference on Empirical Methods in Natural Language Processing (EMNLP)},
  year      = {2023}
}

@article{raptor2024,
  title     = {{RAPTOR}: Recursive Abstractive Processing for Tree-Organized Retrieval},
  author    = {Sarthi, Parth and Abdullah, Salman and Tuli, Aditi and Khanna, Shubh and Goldie, Anna and Manning, Christopher D.},
  journal   = {International Conference on Learning Representations (ICLR)},
  year      = {2024}
}

%%%%%%%%%%%%%%%%%%%%%%%%%%%%%%%%%%%%%%%%%%%%%%%%%%%%%%%%%%%%

\appendix
\onecolumn

\section{Algorithms}
\label{app:algorithms}

This appendix collects the pseudocode for the main \textit{CoCoDA} training loop
(Algorithm~\ref{alg:todo-main}) and its two subroutines: Typed DAG Retrieval
(Algorithm~\ref{alg:hier-retrieve}) and online tool insertion
(Algorithm~\ref{alg:insert-tool}).

\begin{algorithm}[h]
\caption{CoCoDA: Tool-Augmented Planning with Online Co-Evolving Compositional DAG}
\label{alg:todo-main}
\begin{algorithmic}[1]
\Require Base student $\pi_{\text{base}}$, teacher $\mathcal{M}_T$, query stream $\mathcal{Q}$, success threshold $\rho$, reward coefficient $\lambda$, group size $G$, epochs $N$
\Ensure Co-evolved planner $\pi_{\theta^\star}$ and library $\mathcal{L}^\star=(\mathcal{V}^\star,\mathcal{E}^\star,\mathcal{I}^\star)$ with DAG $\mathcal{G}^\star$
\State \textcolor{blue}{$\triangleright$ \textit{Stage 1: Experience-based Tool Distillation (warm-start library)}}
\State $\mathcal{T}^{+}\!\leftarrow\!\textsc{Rollout}(\pi_{\text{base}},\mathcal{Q}_{\text{seed}})$ \Comment{keep only successful trajectories}
\State $\mathcal{V},\mathcal{E},\mathcal{I}\leftarrow\emptyset,\emptyset,\emptyset$
\For{all $\tau^{+}\in\mathcal{T}^{+}$}
  \For{all $t^\star\in\mathcal{M}_T(\tau^{+})$}
    \State \textsc{InsertTool}($t^\star,\mathcal{V},\mathcal{E},\mathcal{I}$) \Comment{Algorithm~\ref{alg:insert-tool}}
  \EndFor
\EndFor
\State $\mathcal{L}\leftarrow(\mathcal{V},\mathcal{E},\mathcal{I})$
\State \textcolor{blue}{$\triangleright$ \textit{Stage 2: Cold-start SFT on warm library}}
\State $\mathcal{D}_{\text{SFT}}\leftarrow\mathcal{M}_T(\mathcal{Q}_{\text{seed}},\mathcal{L})$
\State $\theta\leftarrow\textsc{SFT}(\pi_{\text{base}},\mathcal{D}_{\text{SFT}})$;\ \ $\pi_{\text{ref}}\leftarrow\pi_\theta$
\State \textcolor{blue}{$\triangleright$ \textit{Stage 3: Online Co-Evolution ($\pi_\theta$ and $\mathcal{L}$ update per query)}}
\For{epoch $=1$ to $N$}
  \For{each incoming query $q\in\mathcal{Q}$}
    \State \textcolor{gray}{// (a) Retrieve context from the \emph{current} library}
    \State $\mathcal{C}_q\leftarrow\textsc{TypedDAGRetrieve}(q,\mathcal{I},\mathcal{G})$ \Comment{Algorithm~\ref{alg:hier-retrieve}}
    \State \textcolor{gray}{// (b) Group rollout under current policy}
    \State Sample $\{\tau^{(i)}\}_{i=1}^{G}\sim\pi_\theta(\cdot\mid q,\mathcal{C}_q)$
    \State $R_i\leftarrow R_{\text{res}}(\tau^{(i)})+\lambda\sum_{j}\bigl(\mathrm{flat}(t_j^{(i)})-1\bigr)$
    \State \textcolor{gray}{// (c) Policy update}
    \State Update $\theta$ via GRPO with group-relative advantages over $\{R_i\}$
    \State \textcolor{gray}{// (d) Library update from this query's successful rollouts}
    \State $\mathcal{T}^{+}_q\leftarrow\{\tau^{(i)}\,:\,R_{\text{res}}(\tau^{(i)})\geq\rho\}$
    \If{$\mathcal{T}^{+}_q\neq\emptyset$}
      \For{all $\tau^{+}\in\mathcal{T}^{+}_q$}
        \For{all $t^\star\in\mathcal{M}_T(\tau^{+})$}
          \State \textsc{InsertTool}($t^\star,\mathcal{V},\mathcal{E},\mathcal{I}$) \Comment{merge / add edges / update DAG}
        \EndFor
      \EndFor
      \State $\mathcal{L}\leftarrow(\mathcal{V},\mathcal{E},\mathcal{I})$ \Comment{next query sees the updated library}
    \EndIf
  \EndFor
\EndFor
\State \Return $\pi_\theta,\ \mathcal{L}$
\end{algorithmic}
\end{algorithm}

\begin{algorithm}[h]
\caption{\textsc{TypedDAGRetrieve}: Typed DAG Retrieval over the Compositional Code DAG}
\label{alg:hier-retrieve}
\begin{algorithmic}[1]
\Require Query $q$, DAG $\mathcal{G}=(\mathcal{V},\mathcal{E})$ with per-node records $\{L_1,L_2,L_3,L_4\}$ (signature, description, specification, examples), planner $\pi_\theta$, shortlist size $k_2$
\Ensure Candidate set $\mathcal{C}_q$ annotated with graph-structural features
\State Decompose $q$ into typed sub-goals $\{(T_{\text{in}}^{(j)}\!\to\!T_{\text{out}}^{(j)},\ \text{intent}^{(j)})\}_{j=1}^{J}$
\State $\mathcal{C}_q\leftarrow\emptyset$
\For{$j=1$ to $J$}
  \State $\mathcal{S}_1^{(j)}\!\leftarrow\!\{v\in\mathcal{V}\mid L_1(v)\ \text{unifies with}\ T_{\text{in}}^{(j)}\!\to\!T_{\text{out}}^{(j)}\}$ \Comment{L1 type filter}
  \State $\mathcal{S}_2^{(j)}\!\leftarrow\!\mathrm{Top}\text{-}k_2\bigl\{\pi_\theta\bigl(\text{intent}^{(j)},L_2(v)\bigr) : v\in\mathcal{S}_1^{(j)}\bigr\}$ \Comment{L2 semantic scan}
  \State $\mathcal{S}_3^{(j)}\!\leftarrow\!\{v\in\mathcal{S}_2^{(j)}\mid L_3(v)\ \text{satisfies pre/post-conditions of}\ \text{intent}^{(j)}\}$ \Comment{L3 constraint check}
  \State $v^\star_j\!\leftarrow\!\arg\max_{v\in\mathcal{S}_3^{(j)}}\pi_\theta\bigl(\text{intent}^{(j)},L_4(v)\bigr)$ \Comment{L4 example rerank}
  \State Attach graph features $\bigl(d(v^\star_j),\mathrm{Ch}(v^\star_j),|\mathcal{G}_{v^\star_j}|\bigr)$
  \State $\mathcal{C}_q\leftarrow\mathcal{C}_q\cup\{v^\star_j\}$
\EndFor
\State \Return $\mathcal{C}_q$
\end{algorithmic}
\end{algorithm}

\begin{algorithm}[h]
\caption{\textsc{InsertTool}: Online Insertion of a New Tool into the Compositional DAG}
\label{alg:insert-tool}
\begin{algorithmic}[1]
\Require Candidate tool $t^\star$, vertex set $\mathcal{V}=\mathcal{V}_p\cup\mathcal{V}_c$, edge set $\mathcal{E}$, hierarchical index $\mathcal{I}$
\Ensure Updated $(\mathcal{V},\mathcal{E},\mathcal{I})$ with $t^\star$ inserted as primitive or composite, or unchanged on rejection
\State Validate $t^\star$ against $L_3,L_4$ of each $u\in\mathrm{Ch}(t^\star)$
\If{validation fails \textbf{or} $\mathcal{G}\cup\{t^\star\}$ not acyclic}
  \State \Return \Comment{reject $t^\star$}
\EndIf
\If{$\mathrm{Ch}(t^\star)=\emptyset$}
  \State $\mathcal{V}_p\leftarrow\mathcal{V}_p\cup\{t^\star\}$;\ \ $d(t^\star)\leftarrow 0$ \Comment{primitive}
\Else
  \State $\mathcal{V}_c\leftarrow\mathcal{V}_c\cup\{t^\star\}$;\ \ $d(t^\star)\leftarrow 1+\max_{u\in\mathrm{Ch}(t^\star)}d(u)$ \Comment{composite}
  \State $\mathcal{E}\leftarrow\mathcal{E}\cup\{(t^\star,u)\mid u\in\mathrm{Ch}(t^\star)\}$
\EndIf
\State Build record $\{L_1,L_2,L_3,L_4\}(t^\star)$ and insert into $\mathcal{I}$
\end{algorithmic}
\end{algorithm}

\section{Datasets}
\label{app:datasets}

We evaluate \textit{CoCoDA} on six public benchmarks that span three categories of execution-verifiable reasoning. Table~\ref{tab:dataset_stats} summarizes the split sizes used in our experiments.

\textbf{GSM8K}~\citep{cobbe2021gsm8k} is a dataset of $8.5$K grade-school math word problems that require multi-step arithmetic reasoning. Each problem is paired with a natural-language solution and a final numerical answer, so correctness can be checked by evaluating the Python expression produced by the planner against the reference answer.

\textbf{MATH}~\citep{hendrycks2021math} contains $12.5$K competition-level mathematics problems (algebra, geometry, number theory, precalculus, etc.) with step-by-step solutions and boxed final answers. Compared with GSM8K, it requires longer symbolic derivations, making it a natural stress test for compositional tool use.

\textbf{WikiTableQuestions (WTQ)}~\citep{pasupat2015wikitable} is a table question-answering benchmark over $22$K questions about $2{,}108$ semi-structured Wikipedia tables. Answers are cells, aggregations, or comparisons computable by combining SQL-style table lookups with numerical operators.

\textbf{FinQA}~\citep{chen2021finqa} is a financial QA dataset of $8{,}281$ questions grounded in real earnings reports, where each answer is produced by a short numerical reasoning program over tables and accompanying text. It emphasizes multi-step arithmetic over structured financial data.

\textbf{EvalPlus}~\citep{liu2023evalplus} is an extension of HumanEval and MBPP that augments each problem with a substantially larger and more adversarial test suite, giving a more faithful estimate of functional correctness than the original pass@1. We use its HumanEval+ split for code Task.

\textbf{MBPP}~\citep{austin2021mbpp} (Mostly Basic Python Problems) contains $974$ short Python programming tasks with natural-language descriptions and three unit tests each. Correctness is defined by passing all provided test cases.

All six benchmarks share the property that each task can be verified by executing deterministic code (Python for math and coding, SQL combined with numerical computation for tabular analysis), which makes them well-suited to the execution-grounded reward signal used in our GRPO training.

\section{Baseline Methods}
\label{app:baselines}

We group our baselines into three categories that together cover the main design choices our framework departs from. All baselines are trained or prompted on the same Qwen3 student backbone and receive identical task prompts.

\paragraph{No-tool chain-of-thought baselines.}
\textbf{CoT (student)}~\citep{wei2023chain} prompts the base Qwen3 student with standard chain-of-thought reasoning and no external tool access, establishing the performance attainable from pure parametric reasoning at the student scale. \textbf{CoT (teacher)} applies the same protocol to the Qwen3-32B teacher and is reported as an upper-bound reference.

\paragraph{Fixed-library tool-augmented methods.}
\textbf{ReAct}~\citep{react2023} interleaves reasoning traces and tool calls at inference time; the planner chooses from a fixed tool inventory and is not trained via reinforcement learning. \textbf{ToRL}~\citep{torl2024} trains a tool-integrated reasoning policy with reinforcement learning over a static set of tools, optimizing for end-to-end task reward. \textbf{ReTool}~\citep{retool2025} also trains a tool-integrated reasoner with GRPO-style updates and shares our RL backbone, but keeps the tool inventory frozen throughout training, isolating the contribution of library co-evolution.

\paragraph{Library-learning methods.}
\textbf{CREATOR}~\citep{creator2023} disentangles tool creation from tool use: a teacher model first synthesizes a flat library of Python utilities offline, and the solver later invokes them, without feedback from solver performance to library construction. \textbf{TroVE}~\citep{trove2024} induces a verified function library by abstracting recurring patterns from successful program traces and reuses them as a flat collection, again decoupling library growth from policy optimization.

\begin{table}[h]
\centering
\small
\caption{Statistics of the six evaluation benchmarks.}
\label{tab:dataset_stats}
\begin{tabular}{lllr}
\toprule
\textbf{Dataset} & \textbf{Category} & \textbf{Verification} & \textbf{\#Test} \\
\midrule
GSM8K & Math / Logical & Exec.\ (numeric) & $1{,}319$ \\
MATH & Math / Logical & Exec.\ (symbolic) & $5{,}000$ \\
WikiTableQuestions & Tabular Analysis & SQL + numeric & $4{,}344$ \\
FinQA & Tabular Analysis & Numeric program & $1{,}147$ \\
EvalPlus (HumanEval+) & Code Task & Unit tests & $164$ \\
MBPP & Code Task & Unit tests & $500$ \\
\bottomrule
\end{tabular}
\end{table}

\section{Prompts}
\label{app:prompts}

This appendix lists every prompt used by \textit{CoCoDA}. All prompts are issued
to either the teacher model $\mathcal{M}_T$ (Qwen3-32B, used for tool
abstraction, cold-start SFT generation, and semantic deduplication) or the
student planner $\pi_\theta$ (for task decomposition, retrieval, and
tool-augmented planning). Placeholders in curly braces are substituted at
runtime. We fix temperature $0.2$ and top-$p$ $0.95$ for all auxiliary
(non-rollout) prompts; rollout-time generation settings are reported in
Appendix~\ref{app:hparams}.

\subsection{Teacher Tool-Abstraction Prompt}
\label{app:prompt-abstract}

Used in Stage~1 and Stage~3(d) of Algorithm~\ref{alg:todo-main} to distill a
successful trajectory $\tau^{+}$ into a set of candidate compositional tools
$\{t^\star\}$ whose bodies invoke only tools already in $\mathcal{V}$.

\begin{promptbox}{Teacher Tool-Abstraction Prompt}
[System]
You are a library curator. Given a successful problem-solving trajectory and
the current tool library, propose ONE OR MORE reusable compositional tools
that abstract concrete arguments into typed parameters. Every proposed tool
must (a) be implemented by calling ONLY tools already in the library,
(b) expose a single typed entry point, and (c) come with a four-level
metadata record (L1-L4). Do NOT propose a tool whose body is a single
existing tool call.

[User]
Current library (name :: signature ; one-line description):
{LIBRARY_SIGNATURES_AND_L2}

Successful trajectory (interleaved reasoning, tool calls, and results):
{TRAJECTORY}

Return a JSON list. Each element has fields:
  name   : snake_case identifier
  L1     : "{name} :: {T_in} -> {T_out}; deps=[child_1, ...]"
  L2     : "<one-line description>; tags=[...]"
  L3     : {"pre": "...", "post": "...", "complexity": "O(...)"}
  L4     : [{"in": "...", "out": "..."}, ...]   # >= 2 examples
  body   : Python-style pseudocode calling only existing library tools
If no useful abstraction exists, return [].
\end{promptbox}

\subsection{Teacher Cold-Start SFT Prompt}
\label{app:prompt-sft}

Used once per seed query $q\in\mathcal{Q}_{\text{seed}}$ to generate
demonstrations $\mathcal{D}_{\text{SFT}}$ against the warm library
$\mathcal{L}$.

\begin{promptbox}{Teacher Cold-Start SFT Prompt}
[System]
You are solving a task by planning a short program that invokes tools from
the provided library. Emit a reasoning trace interleaved with tool calls.
Each tool call must use EXACTLY the signature and preconditions declared
below. Prefer higher-level composite tools when they subsume a sequence
of primitive calls. After the final tool call, emit an <answer>...</answer>
block with the final answer.

[User]
Available tools (L1 signature + L2 description + L3 specification):
{TOOL_CARDS}

Query: {QUERY}

Respond with:
<plan> ... </plan>
<calls>
  tool_name(args) -> result   # one per line
</calls>
<answer> ... </answer>
\end{promptbox}

\subsection{Task-Decomposition Prompt (Planner)}
\label{app:prompt-decompose}

Executed by $\pi_\theta$ at the start of every query to emit a sequence of
typed sub-goals consumed by \textsc{TypedDAGRetrieve}
(Algorithm~\ref{alg:hier-retrieve}, line~1).

\begin{promptbox}{Task-Decomposition Prompt}
[System]
Decompose the user query into an ordered list of typed sub-goals. Each
sub-goal is a triple (T_in, T_out, intent) where T_in and T_out are concrete
Python types (int, float, str, List[float], DataFrame, ...) and intent is a
one-sentence description of what the sub-goal should accomplish. Keep the
list short: 1-5 sub-goals.

[User]
Query: {QUERY}

Return JSON:
[
  {"T_in": "...", "T_out": "...", "intent": "..."},
  ...
]
\end{promptbox}

\subsection{L2 Semantic-Ranker Prompt}
\label{app:prompt-l2}

Single forward pass over all L1-surviving candidates for a given sub-goal,
returning a top-$k_2$ shortlist. The planner itself plays the role of the
ranker, so no separate model is introduced.

\begin{promptbox}{L2 Semantic-Ranker Prompt}
[System]
You are a tool-retrieval ranker. Given a sub-goal and a list of candidate
tools with their one-line descriptions, return the indices of the top-K
tools most likely to implement the sub-goal. Consider BOTH semantic match
and domain tags.

[User]
Sub-goal: {INTENT}  (expects {T_IN} -> {T_OUT})
Candidates (idx. name :: sig -- L2):
  0. add :: (float,float)->float -- Add two floats; tags=[arithmetic]
  1. quadratic_expr :: (float,float,float,float)->float -- Evaluate a x^2+b x+c; tags=[algebra,polynomial]
  ...

Return a JSON list of the top-{K2} indices, most relevant first.
\end{promptbox}

\subsection{L3 Constraint-Check Prompt}
\label{app:prompt-l3}

Applied to each candidate that survives the L2 semantic shortlist. The
planner verifies whether the tool's pre/post-conditions and complexity
specification are compatible with the sub-goal; only candidates judged
\texttt{compatible} proceed to L4 reranking.

\begin{promptbox}{L3 Constraint-Check Prompt}
[System]
You are a tool-retrieval specification checker. For each candidate tool, decide
whether its declared pre-conditions, post-conditions, and complexity are
consistent with the sub-goal's input/output types and intent. Reject a
candidate if (i) its pre-condition cannot be satisfied by the sub-goal's
inputs, (ii) its post-condition contradicts the sub-goal's expected output,
or (iii) its complexity exceeds the sub-goal's stated budget (when given).
Otherwise accept it.

[User]
Sub-goal: {INTENT}  (expects {T_IN} -> {T_OUT}; budget={BUDGET})

Candidates (idx. name :: sig -- L3 specification):
  0. add :: (float,float)->float --
     pre="finite inputs"; post="returns x+y"; complexity="O(1)"
  1. quadratic_expr :: (float,float,float,float)->float --
     pre="real coefficients"; post="returns a*x^2+b*x+c"; complexity="O(1)"
  ...

Return a JSON list of objects, one per candidate, in the same order:
[
  {"idx": 0, "verdict": "compatible" | "incompatible", "reason": "..."},
  ...
]
\end{promptbox}

\subsection{L4 Example-Based Reranking Prompt}
\label{app:prompt-l4}

Applied only to candidates that survive the L3 constraint check.

\begin{promptbox}{L4 Example-Based Reranking Prompt}
[System]
You will select the single best tool for a sub-goal by comparing each
candidate's concrete input/output examples to the sub-goal's expected
behaviour. Prefer tools whose examples most closely match the expected
input shape and output semantics.

[User]
Sub-goal: {INTENT}
Expected input shape: {INPUT_SHAPE}
Expected output shape: {OUTPUT_SHAPE}

Candidates (idx. name -- L4 examples):
  0. add -- [(1.0,2.0)->3.0, (-1.5,2.5)->1.0]
  1. quadratic_expr -- [(1,-3,2,2)->0, (2,0,-1,3)->17]
  ...

Return the index of the single best candidate and a one-sentence
justification.
\end{promptbox}

\subsection{Planner System Prompt for Tool-Augmented Rollouts}
\label{app:prompt-rollout}

Conditions every GRPO rollout. The context $\mathcal{C}_q$ returned by
\textsc{TypedDAGRetrieve} is injected under \texttt{\{TOOL\_CARDS\}}, and each
card contains L1, L2, L3, L4, the depth $d(v)$, and the direct children
$\mathrm{Ch}(v)$.

\begin{promptbox}{Planner System Prompt for Tool-Augmented Rollouts}
[System]
You are a tool-using agent. Solve the task by emitting interleaved
<think>...</think>, <call>tool_name(args)</call>, and <obs>...</obs> blocks.
After a <call>, wait for the next <obs> before continuing. When confident,
emit <answer>...</answer> and stop.

Rules:
- Only call tools listed below, using the exact signature.
- When a composite tool subsumes several primitives, prefer the composite.
- Never invent arguments that violate L3 preconditions.
- Do NOT emit code outside tool calls.

[Tools]
{TOOL_CARDS}   # each card: L1 / L2 / L3 / L4 / depth / children

[User]
{QUERY}
\end{promptbox}

\subsection{InsertTool Semantic-Deduplication Prompt}
\label{app:prompt-dedup}

Invoked inside \textsc{InsertTool} (Algorithm~\ref{alg:insert-tool}) when a
candidate $t^{\star}$ passes validation. Returns either a merge decision
against an existing tool or a distinctness verdict.

\begin{promptbox}{InsertTool Semantic-Deduplication Prompt}
[System]
Decide whether a candidate tool is a near-duplicate of any existing tool in
the library. Two tools are near-duplicates iff their L1 signatures unify AND
their L3 specifications are logically equivalent AND their L4 examples agree on
all overlapping inputs. If so, MERGE (return the existing tool id and any
L2/L4 entries worth appending). Otherwise, return DISTINCT.

[User]
Candidate:
  L1: {CAND_L1}
  L2: {CAND_L2}
  L3: {CAND_L3}
  L4: {CAND_L4}

Type-compatible existing tools:
{COMPATIBLE_TOOLS}

Return JSON: {"decision": "MERGE"|"DISTINCT", "merge_into": <id or null>,
              "append_L2": "...", "append_L4": [...]}
\end{promptbox}

\section{Hyperparameters}
\label{app:hparams}

Table~\ref{tab:hparams} lists every hyperparameter used in the main
experiments. All student sizes share the same retrieval, reward, and
generation configuration; only the optimizer learning rate and the
parameter-efficient fine-tuning choice vary across student sizes. The
teacher $\mathcal{M}_T$ is frozen throughout.

\begin{table}[htbp]
\centering
\small
\setlength{\tabcolsep}{6pt}
\renewcommand{\arraystretch}{1.1}
\caption{Hyperparameters of \textit{CoCoDA}. Values without a size qualifier
apply to every student. ``FT'' = full fine-tuning with DeepSpeed ZeRO-3;
``LoRA'' = LoRA adapters on every attention and MLP projection.}
\label{tab:hparams}
\begin{tabular}{l l l}
\toprule
\textbf{Group} & \textbf{Hyperparameter} & \textbf{Value} \\
\midrule
\multirow{6}{*}{GRPO}
 & Group size $G$                        & $8$ \\
 & Clip range $\epsilon$                 & $0.2$ \\
 & KL coefficient                        & $0.01$ (to $\pi_{\text{ref}}$) \\
 & Rollout batch (queries / step)        & $64$ \\
 & Online epochs $N$                     & $3$ \\
 & Advantage normalization               & group mean \& std \\
\midrule
\multirow{3}{*}{Reward}
 & Compositional weight $\lambda$        & $0.20$ \\
 & Success threshold $\rho$              & $0.8$ \\
 & Result reward $R_{\text{res}}$        & $\{0,1\}$ via deterministic verifier \\
\midrule
\multirow{3}{*}{Retrieval}
 & L2 shortlist size $k_2$               & $32$ \\
 & L1 type unification                   & structural (covariant outputs) \\
 & L3 specification check                     & LLM-judged on declared pre/post/complexity (App.~\ref{app:prompt-l3}) \\
\midrule
\multirow{5}{*}{Cold-start SFT}
 & Seed queries $|\mathcal{Q}_{\text{seed}}|$ & $512$ \\
 & Epochs                                & $3$ \\
 & Batch size                            & $64$ \\
 & Warmup ratio                          & $0.03$ \\
 & LR schedule                           & cosine to $0$ \\
\midrule
\multirow{4}{*}{Optimizer}
 & Type                                  & AdamW, $\beta_1=0.9$, $\beta_2=0.95$ \\
 & Weight decay                          & $0.01$ \\
 & Grad clipping                         & $1.0$ \\
 & Mixed precision                       & bf16 \\
\midrule
\multirow{4}{*}{Per-size}
 & 0.6B / 1.7B / 4B                      & FT, lr $=2\!\times\!10^{-5},\,1\!\times\!10^{-5},\,5\!\times\!10^{-6}$ \\
 & 8B                                    & LoRA ($r{=}64,\alpha{=}128$), lr $=1\!\times\!10^{-5}$ \\
 & Max sequence length                   & $4096$ \\
 & Rollouts / query                      & $G{=}8$ \\
\midrule
\multirow{5}{*}{Generation}
 & Student rollout temperature           & $0.7$ \\
 & Student top-$p$                       & $0.95$ \\
 & Student max new tokens                & $2048$ \\
 & Teacher temperature / top-$p$         & $0.2$ / $0.95$ \\
 & Teacher max new tokens                & $4096$ \\
\midrule
\multirow{3}{*}{Library init}
 & Warm library $|\mathcal{V}|$ after Stage~1 & $\approx 180$ \\
 & Primitive set size                    & $42$ (shared across domains) \\
 & Max DAG depth at init                 & $2$ \\
\bottomrule
\end{tabular}
\end{table}

\section{Tool Library Data}
\label{app:library}

We give concrete examples of the four-level tool records introduced in
Section~\ref{sec:hier-index}, a snapshot of a small region of the
Compositional Tool DAG, and end-of-training statistics of the evolved
library on each benchmark.

\subsection{Example Primitive Tool Record}

\begin{verbatim}
name : add
L1   : add :: (float, float) -> float ; deps = []
L2   : "Return the sum of two real numbers." ; tags = [arithmetic, primitive]
L3   : pre  : a, b are finite floats
       post : returns a + b
       complexity : O(1)
L4   : [ (1.0, 2.0)   -> 3.0,
         (-1.5, 2.5)  -> 1.0,
         (0.0, 0.0)   -> 0.0 ]
body : def add(a, b):
           return a + b
\end{verbatim}

\subsection{Example Compositional Tool Record}

The composite \texttt{quadratic\_expr} has depth $d=2$ and children
$\{\texttt{add}, \texttt{mul}, \texttt{pow\_int}\}$.

\begin{verbatim}
name : quadratic_expr
L1   : quadratic_expr :: (float, float, float, float) -> float
       deps = [add, mul, pow_int]
L2   : "Evaluate the quadratic a*x^2 + b*x + c at x."
       tags = [algebra, polynomial, closed-form]
L3   : pre  : a, b, c, x are finite floats
       post : returns a*x^2 + b*x + c
       complexity : O(1)
L4   : [ (1, -3,  2, 2)  -> 0.0,
         (2,  0, -1, 3)  -> 17.0,
         (0,  1,  0, 5)  -> 5.0 ]
body : def quadratic_expr(a, b, c, x):
           return add( add( mul(a, pow_int(x, 2)), mul(b, x) ), c )
\end{verbatim}

A second example, extracted on WikiTableQuestions during online co-evolution,
aggregates a numeric column under a row filter:

\begin{verbatim}
name : sum_where
L1   : sum_where :: (DataFrame, str, Predicate) -> float
       deps = [filter_rows, select_column, sum_list]
L2   : "Sum a numeric column over rows satisfying a predicate."
       tags = [tabular, aggregation]
L3   : pre  : column exists and is numeric; predicate is total on rows
       post : returns sum of column over filtered rows
       complexity : O(n)
L4   : [ (sales_df, "amount", lambda r: r.region=="EU") -> 31420.0, ... ]
body : def sum_where(df, col, pred):
           return sum_list( select_column( filter_rows(df, pred), col ) )
\end{verbatim}

\subsection{Compositional Tool DAG Snapshot (Math Domain)}

Figure~\ref{fig:dag-snippet} sketches a region of the DAG after training.
The figure makes visible how higher-depth composites are layered on top of
primitive arithmetic operators, with shared low-level nodes reused across
multiple composites, which is exactly the reuse pattern that drives the
compositional advantage established in Theorem~\ref{thm:compositional}.

\begin{figure}[ht]
\centering
\includegraphics[width=0.9\linewidth]{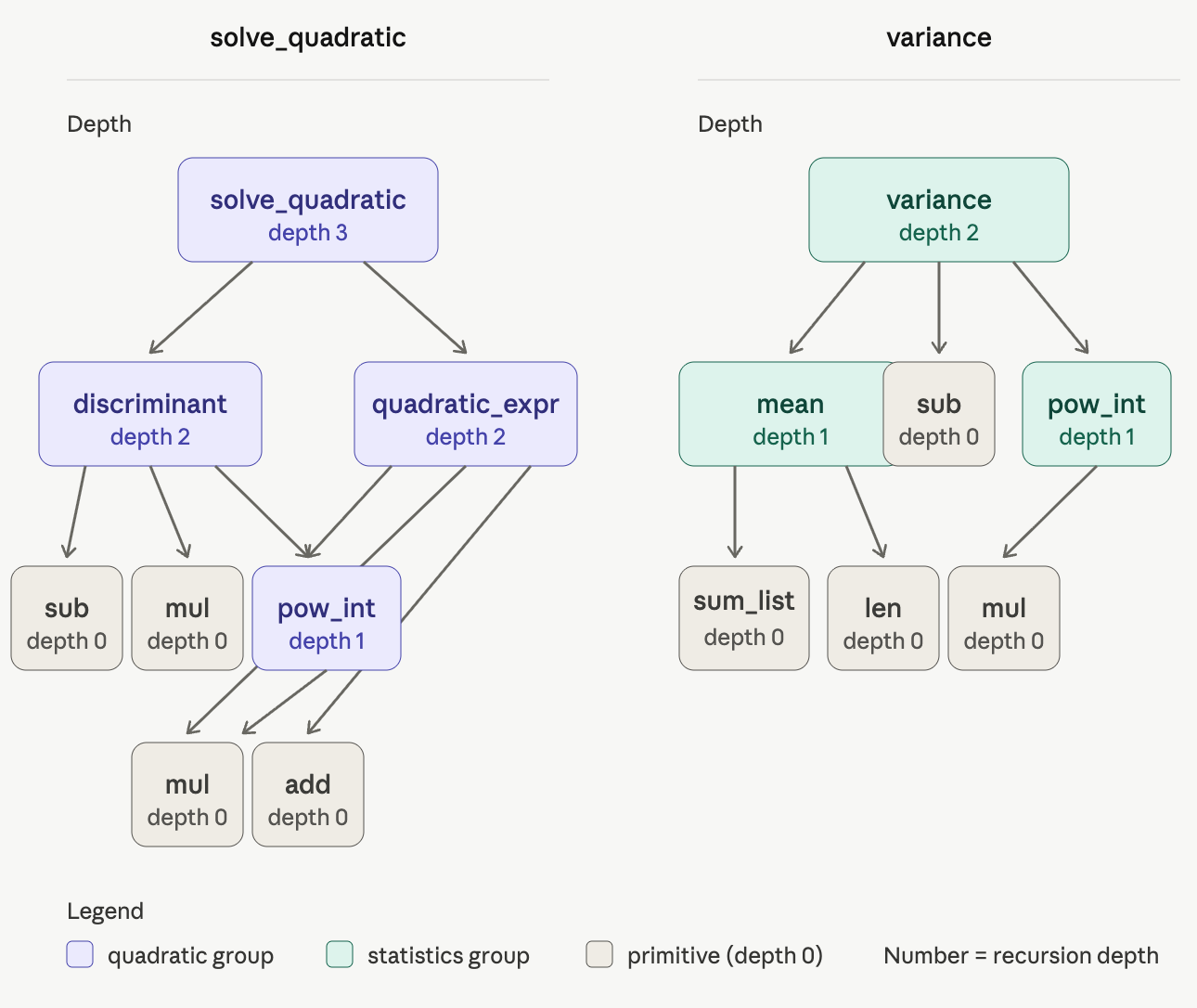}
\caption{A subgraph of the Compositional Tool DAG on the math domain after
training. Edges point from composites to their direct dependencies.}
\label{fig:dag-snippet}
\end{figure}

\subsection{End-of-Training Library Statistics}

Table~\ref{tab:lib-stats} reports the state of $\mathcal{L}$ at the end of
Stage~3 for each benchmark, broken down by primitives, composites,
maximum depth, and mean fan-out.

\begin{table}[htbp]
\centering
\small
\setlength{\tabcolsep}{6pt}
\renewcommand{\arraystretch}{1.1}
\caption{Evolved library statistics at the end of training (4B student).
$|\mathcal{V}_p|$ is fixed at initialization; $|\mathcal{V}_c|$ grows during
co-evolution. Mean fan-out is computed over composites only.}
\label{tab:lib-stats}
\begin{tabular}{l cccc c}
\toprule
\textbf{Benchmark} & $|\mathcal{V}_p|$ & $|\mathcal{V}_c|$ & $\max_v d(v)$ & mean $|\mathrm{Ch}(v)|$ & \# rejected by \textsc{InsertTool} \\
\midrule
GSM8K            & 18 & 291 & 4 & 2.6 & 112 \\
MATH             & 18 & 364 & 4 & 2.8 & 158 \\
WTQ              & 14 & 217 & 3 & 2.3 &  78 \\
FinQA            & 14 & 246 & 4 & 2.5 &  94 \\
EvalPlus         & 22 & 183 & 3 & 2.2 &  61 \\
MBPP             & 22 & 208 & 3 & 2.3 &  73 \\
\bottomrule
\end{tabular}
\end{table}

\section{Dataset Details}
\label{app:dataset-details}

Table~\ref{tab:datasets} summarises the six benchmarks. We use the official
train/test splits for GSM8K, MATH, WTQ, FinQA, and MBPP, and the full
EvalPlus suite (extended HumanEval with hardened unit tests) for code
generation. The training split is used both as the seed set
$\mathcal{Q}_{\text{seed}}$ and as the GRPO query stream $\mathcal{Q}$;
evaluation is reported on the held-out test split.

\begin{table}[htbp]
\centering
\small
\setlength{\tabcolsep}{6pt}
\renewcommand{\arraystretch}{1.1}
\caption{Dataset statistics and verifier used for each benchmark.}
\label{tab:datasets}
\begin{tabular}{l l r r l}
\toprule
\textbf{Dataset} & \textbf{Task} & \textbf{\#train} & \textbf{\#test} & \textbf{Verifier} \\
\midrule
GSM8K      & grade-school math     & 7,473  & 1,319 & Python exec + exact-match on numeric answer \\
MATH       & competition math      & 7,500  & 5,000 & Python exec + SymPy equivalence check \\
WTQ        & table QA              & 14,152 & 4,344 & SQL exec + normalized string/numeric match \\
FinQA      & financial table QA    & 6,251  & 1,147 & Python exec + relative-error $<10^{-3}$ \\
EvalPlus   & code task             & ---    &   164 & Extended HumanEval unit tests (pass@1) \\
MBPP       & code task             &   374  &   500 & Crowd-sourced unit tests (pass@1) \\
\bottomrule
\end{tabular}
\end{table}

\section{Case Studies}
\label{app:cases}

\subsection{Cascaded Retrieval on a GSM8K Query}

\textbf{Query.} \emph{``Lena buys 3 notebooks at \$4 each and 2 pens at
\$1.50 each. How much does she spend in total?''}

\textbf{Task decomposition.}
\begin{verbatim}
[
  {"T_in":"(int,float)", "T_out":"float", "intent":"cost of notebooks"},
  {"T_in":"(int,float)", "T_out":"float", "intent":"cost of pens"},
  {"T_in":"(float,float)","T_out":"float","intent":"total spend"}
]
\end{verbatim}

\textbf{L1 type filter.} $1{,}873$ library tools $\to$ $41$ type-compatible
candidates per sub-goal.

\textbf{L2 shortlist ($k_2=32$).} Top-ranked candidates include
\texttt{mul}, \texttt{linear\_cost}, \texttt{unit\_price\_times\_qty},
\texttt{add}, \texttt{sum\_list}.

\textbf{L3 constraint check.} Discards \texttt{sum\_list} (requires a list
input); retains \texttt{mul}, \texttt{linear\_cost},
\texttt{unit\_price\_times\_qty}, \texttt{add}.

\textbf{L4 rerank.} For the first two sub-goals,
\texttt{unit\_price\_times\_qty} (depth $1$, example
$(3, 4.0) \to 12.0$) outranks \texttt{mul}; for the third,
\texttt{add} wins. Planner emits a two-call trajectory instead of the
four-call primitive-only trajectory, receiving a higher shaped reward
under $R_{\text{comp}}$.

\subsection{A Tool Created from Success and Later Reused}

Early in MATH training, the planner solves \emph{``find the roots of
$x^2-5x+6=0$''} by invoking \texttt{sub}, \texttt{mul}, \texttt{pow\_int},
\texttt{sqrt}, \texttt{add}, \texttt{div}. The teacher abstracts the
trajectory into a new composite \texttt{solve\_quadratic} (depth $3$).
On the \emph{next} query, \emph{``for what $k$ does $x^2+kx+9=0$ have a
double root?''}, \textsc{TypedDAGRetrieve} surfaces
\texttt{solve\_quadratic} at the L4 stage, and the planner solves the
problem in two calls rather than seven. This illustrates the per-query
co-evolution claimed in Theorem~\ref{thm:monotone}.

\subsection{Rejected Candidates}

Three representative \textsc{InsertTool} rejections observed during
training:

\begin{description}
\item[Spec failure.] Candidate \texttt{safe\_div} was proposed with
postcondition ``returns $a/b$ for any $b$'' but its body dispatches to
\texttt{div}, whose precondition requires $b\neq 0$. Validation against
$L_3(\texttt{div})$ fails and the tool is rejected.
\item[Cycle.] On MBPP the teacher proposed a composite
\texttt{memoize\_factorial} that, after substitution, invokes
\texttt{fact\_rec} which already depends on \texttt{memoize\_factorial}.
The acyclicity check at line~3 of Algorithm~\ref{alg:insert-tool}
rejects the insertion.
\item[Near-duplicate merge.] A WTQ rollout produced \texttt{column\_sum\_if},
whose L1/L3/L4 agree with the existing \texttt{sum\_where}. The
deduplication prompt (Appendix~\ref{app:prompt-dedup}) returns
\texttt{MERGE}; the candidate's additional L4 examples are appended to
\texttt{sum\_where} and no new vertex is created.
\end{description}

\section{Compute and Reproducibility}
\label{app:compute}

All experiments are run on a single node with $4\times$ NVIDIA H200 80GB
GPUs, $2\times$ AMD EPYC 9654 CPUs, and 1.5TB RAM. Rollouts use vLLM
($0.6.3$) with tensor-parallel $=4$; optimization uses
PyTorch $2.5$ and DeepSpeed ZeRO-3 for FT runs and PEFT $0.13$ for the LoRA
run. The RL loop is implemented on top of the verl framework.

Table~\ref{tab:compute} reports end-to-end wall-clock training time and
GPU-hours per student size. Latency numbers in
Figure~\ref{fig:scaling} are measured with a warm vLLM cache at batch
size $1$ and averaged over $256$ held-out queries per benchmark; error
bars are one standard deviation over $3$ seeds.

\begin{table}[htbp]
\centering
\small
\setlength{\tabcolsep}{6pt}
\renewcommand{\arraystretch}{1.1}
\caption{End-to-end training cost of \textit{CoCoDA} on $4\times$H200.}
\label{tab:compute}
\begin{tabular}{l r r r}
\toprule
\textbf{Student} & \textbf{Stage 1+2 (h)} & \textbf{Stage 3 / epoch (h)} & \textbf{GPU-hours (total)} \\
\midrule
0.6B & 1.2 &  2.1 & 35 \\
1.7B & 1.8 &  3.4 & 55 \\
4B   & 2.9 &  5.6 & 87 \\
8B   & 3.4 &  6.8 (LoRA) & 98 \\
\bottomrule
\end{tabular}
\end{table}

All main-table numbers are averaged over $3$ seeds ($\{13,42,2026\}$) and
the standard deviation is at most $0.4$ on GSM8K/MATH/WTQ/FinQA and at most
$0.7$ on EvalPlus/MBPP.

\section{Ablation Implementation Details}
\label{app:ablations}

Each ablated variant replaces exactly one component with a ``natural
non-structural counterpart'' and re-trains the 4B student under identical
optimizer, reward, and generation settings (Appendix~\ref{app:hparams}).
Concretely:

\begin{description}
\item[w/o CTD (flat library).] The graph $\mathcal{G}$ is discarded. Every
tool, including would-be composites, is stored as an independent entry;
edges and depths are unavailable at retrieval time, and the compositional
reward term is set to $0$ (i.e., $\mathrm{flat}(t)=1$ and hence
$\Phi(t)=0$ for all $t$) so that only $R_{\text{res}}$ remains as
reward. \textsc{InsertTool} still performs acyclicity-free validation
and deduplication.

\item[w/o TDR (flat dense retrieval).] All four annotation levels are
concatenated into a single text blob per node and embedded once with
\texttt{BAAI/bge-large-en-v1.5}. Retrieval is top-$k$ ($k{=}8$) cosine
similarity to the query embedding, equivalent to what text-style
hierarchical RAG and similarity-clustered skill memories collapse to
when code-specific structure is removed. The Compositional Tool DAG and
the graph-aware reward are preserved.

\item[w/o GAR (exec-only reward).] The reward is reduced to
$R(\tau)=R_{\text{res}}(\tau)$. The DAG and Typed DAG Retrieval are
unchanged, but the planner receives no structural pressure. All other
components, including library updates, are unchanged.

\item[w/o CTD + TDR.] Flat library \emph{and} flat dense retrieval; compositional
reward set to $0$.

\item[w/o CTD + TDR + GAR.] All three ablations compose; equivalent to a
GRPO-trained ReTool-style planner over a flat pool of tools collected from
the warm-start phase.
\end{description}

The reward coefficient $\lambda$ is held at its main-experiment value
($0.20$) rather than re-tuned per ablation, so that the ablation
isolates the \emph{mechanism} rather than the coefficient schedule. A
separate re-tuned version produced accuracies within $0.3$ points of
the values reported in Table~\ref{tab:ablation}.

\section{Sensitivity Sweeps}
\label{app:sensitivity}

\paragraph{Reward coefficient $\lambda$.}
Table~\ref{tab:sens-lambda} sweeps the compositional weight $\lambda$ on
GSM8K with the 4B student. Accuracy is largely insensitive in the range
$[0.1, 0.3]$ and degrades outside it: too small a $\lambda$ removes
structural pressure, while too large a $\lambda$ biases rollouts toward
deep composites even when primitives would suffice, raising the rate of
$R_{\text{res}}=0$ rollouts.

\begin{table}[htbp]
\centering
\small
\setlength{\tabcolsep}{6pt}
\renewcommand{\arraystretch}{1.1}
\caption{Sensitivity of GSM8K accuracy (4B student) to the compositional
reward weight $\lambda$.}
\label{tab:sens-lambda}
\begin{tabular}{l cccccc}
\toprule
$\lambda$ & $0.00$ & $0.05$ & $0.10$ & $0.20$ & $0.30$ & $0.50$ \\
\midrule
GSM8K (\%) & 91.32 & 92.07 & 92.45 & \textbf{92.64} & 92.41 & 90.98 \\
\bottomrule
\end{tabular}
\end{table}

\paragraph{Shortlist size $k_2$.}
Varying $k_2\in\{8,16,32,64,128\}$ at a fixed 4B student changes GSM8K
accuracy by at most $0.6$ points; retrieval latency grows nearly linearly
in $k_2$. We fix $k_2=32$ as the knee of the accuracy--cost frontier.

\paragraph{Success threshold $\rho$.}
Setting $\rho$ too low ($\le 0.4$) admits noisy compositions and raises
the \textsc{InsertTool} rejection rate at later queries; setting
$\rho=1.0$ starves the library in early training. The value $\rho=0.8$
used in the main experiments balances library growth and purity: on MATH,
it yields $364$ retained composites vs.\ $419$ at $\rho=0.5$ (of which
$112$ are later overwritten by duplicates) and $207$ at $\rho=1.0$.

%%%%%%%%%%%%%%%%%%%%%%%%%%%%%%%%%%%%%%%%%%%%%%%%%%%%%%%%%%%%

\section{Theoretical Proofs}
\label{sec:proofs}
\label{app:proofs}

In this section, we state the assumptions underlying our analysis and provide
the detailed proofs for the theoretical results stated in
Section~\ref{sec:theory}.

\subsection{Assumptions}

\begin{assumption}[Level Cost Ordering]
\label{asm:level-cost}
For every tool $v\in\mathcal{V}$, the per-tool token costs satisfy
$c_1(v)\ll c_2(v)\ll c_3(v)$ and $c_2(v)\ll c_4(v)$ (no order is imposed
between $c_3(v)$ and $c_4(v)$). Moreover, there exist constants
$0<\alpha_1,\alpha_2,\alpha_3<1$ such that the L1 type filter retains at most
an $\alpha_1$-fraction of $\mathcal{V}$, the L2 semantic scan retains at most
an $\alpha_2$-fraction of its input, and the L3 constraint check retains at
most an $\alpha_3$-fraction of its input.
\end{assumption}

\begin{assumption}[Bounded Rewards and Verifier]
\label{asm:verifier}
The result reward $R_{\text{res}}(\tau)\in[0,1]$ is determined by a
deterministic verifier, and candidate tools abstracted from any successful
trajectory $\tau^{+}$ that pass \textsc{InsertTool}'s spec and acyclicity
checks are preserved in $\mathcal{L}$ across subsequent updates.
\end{assumption}

\begin{assumption}[Retrieval Consistency]
\label{asm:retrieval}
For every query $q$ and any two libraries $\mathcal{L}\subseteq\mathcal{L}'$
(meaning $\mathcal{V}\subseteq\mathcal{V}'$ and $\mathcal{E}\subseteq\mathcal{E}'$),
\textsc{TypedDAGRetrieve} satisfies
$\mathcal{C}_q(\mathcal{L})\subseteq\mathcal{C}_q(\mathcal{L}')$ up to the
shortlist cap $k_2$: either every tool in $\mathcal{C}_q(\mathcal{L})$ also
appears in $\mathcal{C}_q(\mathcal{L}')$, or any tool it displaces is itself
a composite that dominates it in the L4 reranker score. Intuitively, new
library entries cannot crowd out previously retrieved tools unless they are
strictly better substitutes.
\end{assumption}

\subsection{Proof of Theorem~\ref{thm:retrieval}}

\begin{proof}
Let $p_1(v)=\Pr[v\in\mathcal{S}_1]$, $p_2(v)=\Pr[v\in\mathcal{S}_2]$, and
$p_3(v)=\Pr[v\in\mathcal{S}_3]$ be the marginal survival probabilities of
tool $v$ after stages L1, L2, and L3. Write
$\bar c_\ell = \tfrac{1}{n}\sum_{v\in\mathcal{V}}c_\ell(v)$ for the mean
level-$\ell$ cost and $c_\ell^{\max}=\max_{v\in\mathcal{V}}c_\ell(v)$ for the
maximum. By linearity of expectation,
\begin{equation}\label{eq:hier-decomp}
C_{\text{hier}}
= \sum_{v\in\mathcal{V}}c_1(v)
+ \sum_{v\in\mathcal{V}}p_1(v)\,c_2(v)
+ \sum_{v\in\mathcal{V}}p_2(v)\,c_3(v)
+ \sum_{v\in\mathcal{V}}p_3(v)\,c_4(v).
\end{equation}
Assumption~\ref{asm:level-cost} states
$\sum_v p_1(v)=\mathbb{E}[|\mathcal{S}_1|]\le\alpha_1 n$,
$\sum_v p_2(v)=\mathbb{E}[|\mathcal{S}_2|]\le\alpha_2\mathbb{E}[|\mathcal{S}_1|]\le\alpha_1\alpha_2 n$,
and $\sum_v p_3(v)=\mathbb{E}[|\mathcal{S}_3|]\le\alpha_3\mathbb{E}[|\mathcal{S}_2|]\le\alpha_1\alpha_2\alpha_3 n$.
Applying H\"older's inequality
($\sum_v p_\ell(v)\,c_\ell(v)\le c_\ell^{\max}\sum_v p_\ell(v)$) to the last
three terms of~\eqref{eq:hier-decomp} gives
\begin{equation}\label{eq:hier-ub}
C_{\text{hier}}
\le n\bar c_1 + \alpha_1 n\,c_2^{\max} + \alpha_1\alpha_2 n\,c_3^{\max}
  + \alpha_1\alpha_2\alpha_3 n\,c_4^{\max}.
\end{equation}
Since $\alpha_3\le 1$ we have
$\alpha_1\alpha_2\alpha_3 n\,c_4^{\max}\le\alpha_1\alpha_2 n\,c_4^{\max}$, so
\begin{equation}\label{eq:hier-ub-loose}
C_{\text{hier}}
\le n\bar c_1 + \alpha_1 n\,c_2^{\max} + \alpha_1\alpha_2 n\,(c_3^{\max}+c_4^{\max}).
\end{equation}
Rewriting \eqref{eq:hier-ub-loose} by adding and subtracting
$\alpha_1\alpha_2 n(\bar c_1+c_2^{\max})$ and using
$C_{\text{flat}}=n(\bar c_1+\bar c_2+\bar c_3+\bar c_4)$,
\begin{equation}\label{eq:hier-rearr}
C_{\text{hier}}
\le \alpha_1\alpha_2\,C_{\text{flat}}
+ n\bar c_1(1-\alpha_1\alpha_2)
+ n\,c_2^{\max}(\alpha_1-\alpha_1\alpha_2)
+ \alpha_1\alpha_2 n\bigl[(c_3^{\max}+c_4^{\max})-(\bar c_3+\bar c_4)\bigr].
\end{equation}
The bracketed term in~\eqref{eq:hier-rearr} is non-negative but, under the
mild regularity $c_\ell^{\max}=\mathcal{O}(\bar c_\ell)$ (bounded per-level
dispersion), is of the same order as $\alpha_1\alpha_2 C_{\text{flat}}$ and
thus absorbed into the leading term.
Assumption~\ref{asm:level-cost} additionally gives
$\bar c_1/(\bar c_3+\bar c_4)\to 0$ and
$c_2^{\max}/(\bar c_3+\bar c_4)\to 0$, so
$n\bar c_1,\,n c_2^{\max}=o\bigl(n(\bar c_3+\bar c_4)\bigr)=o(C_{\text{flat}})$.
Substituting into~\eqref{eq:hier-rearr} yields
\[
C_{\text{hier}}\;\le\;\alpha_1\alpha_2\,C_{\text{flat}}+o(C_{\text{flat}}).
\]
Since $\alpha_1\alpha_2<1$ by assumption, we conclude
$\limsup_{n\to\infty}C_{\text{hier}}/C_{\text{flat}}\le\alpha_1\alpha_2<1$,
which is strictly less than unity and in particular
$C_{\text{hier}}=o(C_{\text{flat}})$ when
$\max(\bar c_1,c_2^{\max})=o(\bar c_3+\bar c_4)$.
\end{proof}

\subsection{Proof of Corollary~\ref{cor:retrieval-time}}

\begin{proof}
We show that the expected wall-clock cost of \textsc{TypedDAGRetrieve} grows
sublinearly in the library size $n=|\mathcal{V}|$, in contrast with the flat
baseline whose cost is $\Theta(n)$.

\emph{Implementation of L1.} The L1 record of each tool $v$ is a typed
signature $T_{\text{in}}(v)\!\to\!T_{\text{out}}(v)$ drawn from a finite type
alphabet $\mathcal{T}$. We maintain an inverted index
$\mathcal{H}:\mathcal{T}\!\times\!\mathcal{T}\!\to\!2^{\mathcal{V}}$ that maps
each signature pair $(T_{\text{in}},T_{\text{out}})$ to the set of tools with
that signature; \textsc{InsertTool} updates $\mathcal{H}$ in $\mathcal{O}(1)$
amortized time per insertion (Corollary~\ref{cor:dag}). Given a sub-goal
signature $(T_{\text{in}}^{(j)}\!\to\!T_{\text{out}}^{(j)})$, the L1 candidate
set $\mathcal{S}_1^{(j)}$ is recovered by hash lookup followed by an optional
unification pass over polymorphic signatures organized in a balanced type
trie of height $\mathcal{O}(\log n)$. Hence the per-query L1 time is
$\mathcal{O}\bigl(\tau_1(\log n + |\mathcal{S}_1^{(j)}|)\bigr)$ rather than
$\Theta(\tau_1 n)$ as in an exhaustive scan.

\emph{Cost decomposition.} Let $|\mathcal{S}_\ell|$ denote the size of the
surviving candidate set after stage $L_\ell$. By Algorithm~\ref{alg:hier-retrieve},
$L_2$ is applied to $\mathcal{S}_1$, $L_3$ to a top-$k_2$ shortlist of size at
most $k_2$, and $L_4$ to $\mathcal{S}_3\subseteq\mathcal{S}_2$ with
$|\mathcal{S}_2|\le k_2$. Summing per-level costs and taking expectations,
\begin{equation}\label{eq:time-decomp}
T_{\text{hier}}
\;\le\;\tau_1\bigl(\log n + \mathbb{E}|\mathcal{S}_1|\bigr)
+\tau_2\,\mathbb{E}|\mathcal{S}_1|
+\tau_3\,k_2
+\tau_4\,\mathbb{E}|\mathcal{S}_3|.
\end{equation}

\emph{Bounding the survivor sets.} By Assumption~\ref{asm:level-cost},
$\mathbb{E}|\mathcal{S}_1|\le\alpha_1 n$. The L2 stage truncates to a shortlist
of fixed size $k_2$, independent of $n$, so
$\mathbb{E}|\mathcal{S}_2|\le\min(k_2,\alpha_1\alpha_2 n)\le k_2$ and
$\mathbb{E}|\mathcal{S}_3|\le\alpha_3 k_2$. Substituting into
\eqref{eq:time-decomp},
\begin{equation}\label{eq:time-bound}
T_{\text{hier}}
\;\le\;\tau_1\log n + (\tau_1+\tau_2)\alpha_1 n
+(\tau_3+\alpha_3\tau_4)\,k_2.
\end{equation}

\emph{Sublinear regime.} Inequality~\eqref{eq:time-bound} already shows that
the LLM-bound stages $L_2,L_3,L_4$ contribute time independent of $n$ (since
$k_2$ is a fixed shortlist budget). The only $n$-dependent term is the L1
scan $\tau_1\alpha_1 n$. Two cases are of interest:
\begin{enumerate}
\item \emph{Indexed L1.} When the inverted index $\mathcal{H}$ is used and the
sub-goal signature is fully concrete, the unification pass touches only the
matching bucket and the trie path, giving $\mathbb{E}|\mathcal{S}_1|=
\mathcal{O}(k_2)$ and the L1 cost collapses to $\mathcal{O}(\tau_1\log n)$.
Combined with \eqref{eq:time-bound},
\[
T_{\text{hier}} \;=\; \mathcal{O}\bigl(\tau_1\log n
+(\tau_2+\tau_3+\tau_4)\,k_2\bigr),
\]
which is the bound stated in the corollary and is sublinear ($o(n)$) in $n$.
\item \emph{Exhaustive L1.} If $\mathcal{H}$ is not used and L1 is implemented
as a flat scan, the L1 term becomes $\tau_1\alpha_1 n$. Because $\tau_1$ is the
cost of a constant-size symbolic type check (no LLM call) while $\tau_2,\tau_3,\tau_4$
are LLM-call costs, $\tau_1/\tau_\ell\to 0$ for $\ell\ge 2$ in the relevant
regime. The flat baseline incurs the LLM cost on every tool, so
$T_{\text{flat}} = \Theta\bigl((\tau_2+\tau_3+\tau_4)\,n\bigr)$, and the ratio
\[
\frac{T_{\text{hier}}}{T_{\text{flat}}}
\;\le\;\frac{\tau_1\alpha_1}{\tau_2+\tau_3+\tau_4}
+\frac{(\tau_3+\alpha_3\tau_4)k_2}{(\tau_2+\tau_3+\tau_4)n}
\;\xrightarrow[n\to\infty]{}\;\frac{\tau_1\alpha_1}{\tau_2+\tau_3+\tau_4}\;\ll\;1,
\]
so even without the inverted index the dominant LLM-bound work is reduced from
$\Theta(n)$ to $\mathcal{O}(k_2)$, and the residual linear term carries a
constant $\tau_1\alpha_1$ that is orders of magnitude smaller than the flat
constant.
\end{enumerate}

\emph{Lower bound for the flat baseline.} For
$\mathcal{M}_{\text{flat}}$ (Theorem~\ref{thm:retrieval}), the single ranking pass
materializes $\{L_1,L_2,L_3,L_4\}(v)$ for every $v\in\mathcal{V}$ and applies
the planner to each, so $T_{\text{flat}}\ge(\tau_2+\tau_3+\tau_4)\,n=\Theta(n)$,
matching the upper bound. Combining with case (i) yields
$T_{\text{hier}}/T_{\text{flat}}=\mathcal{O}\bigl((\log n + k_2)/n\bigr)\to 0$,
i.e.\ \textsc{TypedDAGRetrieve} achieves strictly sublinear retrieval time while
the flat baseline is necessarily linear.
\end{proof}

\subsection{Proof of Theorem~\ref{thm:compositional}}

\begin{proof}
We compute both sides of the reward gap in closed form and then propagate the
inequality through the GRPO advantage and the clipped policy-gradient update.

\emph{Step 1: reward gap.}
By construction $\tau_p$ invokes only primitives, so $\Phi(t_i)=0$ for
every $i$ and $R_{\text{comp}}(\tau_p)=\sum_{i=1}^{T_p}\Phi(t_i)=0$. In
$\tau_c$, the contiguous block of $m$ primitives in $\tau_p$ is
replaced by a single composite call $t^\star$ with saved-call count
$\Phi(t^\star)=\mathrm{flat}(t^\star)-1$, while the remaining $T_p-m$
calls are the unchanged primitives with $\Phi=0$. Hence
$R_{\text{comp}}(\tau_c)=\Phi(t^\star)$. Subtracting and using
$R(\tau)=R_{\text{res}}(\tau)+\lambda R_{\text{comp}}(\tau)$ with
$R_{\text{res}}(\tau_p)=R_{\text{res}}(\tau_c)$,
\begin{equation}\label{eq:reward-gap}
R(\tau_c)-R(\tau_p) \;=\; \lambda\,\Phi(t^\star) \;=\;
\lambda\bigl[\mathrm{flat}(t^\star)-1\bigr].
\end{equation}
The faithful-executor lower bound $\mathrm{flat}(t^\star)\ge m\ge 2$
(replacing a length-$m$ primitive subroutine cannot reduce the number
of primitive invocations actually executed) gives the strict bound
$R(\tau_c)-R(\tau_p)\ge\lambda(m-1)>0$. Equivalently, writing the gap
in terms of compression depth,
\begin{equation}\label{eq:reward-gap-2}
R(\tau_c)-R(\tau_p) \;=\; \lambda\bigl[\Phi(t^\star)-(m-1)\bigr]
\;+\;\lambda(m-1)\;=\;\lambda\bigl[\mathrm{flat}(t^\star)-m\bigr]+\lambda(m-1),
\end{equation}
where the first summand is the strictly positive ``internal-reuse
bonus'' (zero exactly when $t^\star$ is a flat renaming of the $m$
primitives) and the second is the unconditional saving from collapsing
$m$ tokens into one.

\emph{Step 2: group-relative advantage.} For a group
$\{\tau^{(1)},\ldots,\tau^{(G)}\}$ sampled from the current policy, GRPO
normalizes rewards by the empirical group mean $\mu$ and standard deviation
$\sigma>0$: $A(\tau^{(i)}) = (R(\tau^{(i)})-\mu)/\sigma$. The map
$R\mapsto (R-\mu)/\sigma$ is strictly increasing for any fixed $\mu$ and
$\sigma>0$, so $R(\tau_c)>R(\tau_p)$ implies
$A(\tau_c)-A(\tau_p) = (R(\tau_c)-R(\tau_p))/\sigma>0$.

\emph{Step 3: policy gradient.} Let $\rho_\theta(\tau) =
\pi_\theta(\tau)/\pi_{\theta_{\text{old}}}(\tau)$ be the importance ratio.
The clipped GRPO objective is
$\mathcal{J}(\theta)=\mathbb{E}\bigl[\min\!\bigl(\rho_\theta A,\,
\mathrm{clip}(\rho_\theta,1-\epsilon,1+\epsilon)\,A\bigr)\bigr]$, whose
gradient at $\theta_{\text{old}}$ (where $\rho_{\theta_{\text{old}}}\!=\!1$
lies in the interior of $[1-\epsilon,1+\epsilon]$) reduces to
$\nabla_\theta\mathcal{J}(\theta)\big|_{\theta_{\text{old}}}
= \mathbb{E}[A(\tau)\nabla_\theta\log\pi_\theta(\tau)]$.
For a single GRPO step with step size $\eta>0$ along this gradient, a
first-order expansion of $\log\pi_\theta(\tau)$ around $\theta_{\text{old}}$
gives
\[
\Delta_\theta\bigl[\log\pi_\theta(\tau_c)-\log\pi_\theta(\tau_p)\bigr]
= \eta\,\bigl\langle A(\tau_c)g(\tau_c)-A(\tau_p)g(\tau_p),\;
g(\tau_c)-g(\tau_p)\bigr\rangle + O(\eta^2),
\]
where $g(\tau)=\nabla_\theta\log\pi_{\theta_{\text{old}}}(\tau)$. Taking
expectation over groups and using
$A(\tau_c)>A(\tau_p)$ with the Fisher-information positivity
$\mathbb{E}\bigl[\langle g(\tau_c)-g(\tau_p),\, g(\tau_c)-g(\tau_p)\rangle\bigr]\ge 0$,
the leading term is non-negative and is strictly positive whenever $\tau_c$
and $\tau_p$ are not policy-equivalent (i.e.\ $g(\tau_c)\neq g(\tau_p)$).
Thus, for small enough $\eta$, a single GRPO update strictly increases the
expected log-odds of $\tau_c$ relative to $\tau_p$, as claimed.
\end{proof}

\subsection{Proof of Theorem~\ref{thm:monotone}}

\begin{proof}
Fix $k\ge 0$ and telescope the one-step gap as
\begin{equation}\label{eq:telescope}
\underbrace{J(\pi_{\theta_{k+1}},\mathcal{L}_{k+1})-J(\pi_{\theta_k},\mathcal{L}_k)}_{\Delta_k}
=\underbrace{J(\pi_{\theta_{k+1}},\mathcal{L}_k)-J(\pi_{\theta_k},\mathcal{L}_k)}_{\Delta_k^{(\pi)}}
+\underbrace{J(\pi_{\theta_{k+1}},\mathcal{L}_{k+1})-J(\pi_{\theta_{k+1}},\mathcal{L}_k)}_{\Delta_k^{(\mathcal{L})}}.
\end{equation}
We establish $\Delta_k^{(\pi)}\ge 0$ and $\Delta_k^{(\mathcal{L})}\ge 0$ separately.

\emph{Policy term.} With the library held fixed at $\mathcal{L}_k$, the GRPO
update on $\pi_\theta$ is a trust-region-style update on the shaped-reward
objective $J(\cdot,\mathcal{L}_k)$. By the standard Kakade--Langford
monotonic-improvement lemma applied to the clipped GRPO surrogate
(equivalent to the PPO surrogate with a group-baseline
advantage~\citep{deepseekmath2024}), we have for any $\theta'$ in the trust
region of $\theta_{\text{old}}$,
\[
J(\pi_{\theta'},\mathcal{L}_k)-J(\pi_{\theta_{\text{old}}},\mathcal{L}_k)
\ge \mathcal{L}^{\text{CLIP}}_{\theta_{\text{old}}}(\theta')
- \frac{2\epsilon\,\gamma}{(1-\gamma)^2}\,D_{\mathrm{KL}}^{\max}(\pi_{\theta'},\pi_{\theta_{\text{old}}}),
\]
where $\mathcal{L}^{\text{CLIP}}$ is the clipped surrogate, $\gamma$ the
effective discount of the trajectory reward, and $\epsilon$ is the
advantage bound. The GRPO update selects $\theta_{k+1}$ as a stochastic
ascent step on $\mathcal{L}^{\text{CLIP}}$; for step sizes
$\eta\le\bar\eta_1$ with
$\bar\eta_1=\tfrac{(1-\gamma)^2}{2\epsilon\gamma L_{\text{KL}}}$ (where
$L_{\text{KL}}$ is the local Lipschitz constant of
$D_{\mathrm{KL}}^{\max}$), the surrogate gain dominates the KL penalty and
$\Delta_k^{(\pi)}\ge 0$ holds in expectation. This is the standard
monotonic-improvement guarantee; we reproduce it here only to make the
step-size requirement on $\bar\eta$ explicit.

\emph{Library term.} Let $\mathcal{L}_{k+1}=(\mathcal{V}_k\cup\Delta\mathcal{V},
\mathcal{E}_k\cup\Delta\mathcal{E},\mathcal{I}_{k+1})$ where
$\Delta\mathcal{V}$ is the set of new composites inserted from the $k$-th
query's successful rollouts. By Assumption~\ref{asm:verifier},
$\mathcal{V}_k\subseteq\mathcal{V}_{k+1}$
and every previously admitted tool remains available, so every primitive
trajectory feasible under $\mathcal{L}_k$ is also feasible under
$\mathcal{L}_{k+1}$ with identical $R_{\text{res}}$ (the verifier is
deterministic and depends only on the program, not the library). Fix a query
$q$ and let $\mathcal{C}_q^{(k)}\!=\!\mathcal{C}_q(\mathcal{L}_k)$,
$\mathcal{C}_q^{(k+1)}\!=\!\mathcal{C}_q(\mathcal{L}_{k+1})$.
We partition trajectories under
$\pi_{\theta_{k+1}}(\cdot\mid q,\mathcal{C}_q^{(k+1)})$ into two classes.
\textit{(a) Legacy trajectories} that use only tools in
$\mathcal{V}_k\cap\mathcal{C}_q^{(k+1)}$: by Assumption~\ref{asm:retrieval},
the legacy subset of the new context is at least as rich as
$\mathcal{C}_q^{(k)}$ restricted to retained tools, so the conditional
expected shaped return over legacy trajectories is bounded below by
$\mathbb{E}_{\tau\sim\pi_{\theta_{k+1}}(\cdot\mid q,\mathcal{C}_q^{(k)})}[R(\tau)]$.
\textit{(b) Composite-adopting trajectories} that invoke at least one tool
$t^\star\!\in\!\Delta\mathcal{V}$: by construction of \textsc{InsertTool},
each $t^\star$ is validated against the L3 pre/post-conditions of its
children, so any trajectory that invokes $t^\star$ has $R_{\text{res}}$ at
least as high as the corresponding primitive-expanded trajectory it
replaces; moreover, by Theorem~\ref{thm:compositional} applied to each
composite substitution of length $m\ge 2$, the shaped reward satisfies
$R(\tau_c)\ge R(\tau_p)+\lambda\,\Phi(t^\star)\ge R(\tau_p)+\lambda(m-1)>R(\tau_p)$.
Taking expectation over the mixing distribution induced by
$\pi_{\theta_{k+1}}$ and summing over classes (a) and (b) yields
$\mathbb{E}_{\tau\sim\pi_{\theta_{k+1}}(\cdot\mid q,\mathcal{C}_q^{(k+1)})}[R(\tau)]
\ge \mathbb{E}_{\tau\sim\pi_{\theta_{k+1}}(\cdot\mid q,\mathcal{C}_q^{(k)})}[R(\tau)]$
for every $q$. Integrating over $q\sim\mathcal{Q}$ gives
$\Delta_k^{(\mathcal{L})}\ge 0$.

Combining the two bounds in \eqref{eq:telescope} yields
$\Delta_k\ge 0$ for every $k$, and choosing
$\bar\eta\le\bar\eta_1$ makes this hold for all $\eta\in(0,\bar\eta]$.
Induction on $k$ completes the proof.
\end{proof}

\subsection{Proof of Corollary~\ref{cor:dag}}

\begin{proof}
Acyclicity follows by induction on $k$. The base case $\mathcal{G}_0$ is empty
(or consists only of primitives with no outgoing edges) and is trivially
acyclic. Suppose $\mathcal{G}_k$ is acyclic; \textsc{InsertTool} accepts a
candidate $t^\star$ only if
$\mathcal{G}_k\cup\{t^\star,\,\mathrm{edges}(t^\star,\mathrm{Ch}(t^\star))\}$
is acyclic (line 3 of Algorithm~\ref{alg:insert-tool}), so
$\mathcal{G}_{k+1}$ is acyclic. The depth recurrence is the definition
$d(v)=1+\max_{u\in\mathrm{Ch}(v)}d(u)$ for composites (and $d(v)=0$ for
primitives), which the algorithm sets at insertion time and never revises.
For the logarithmic growth, every composite of depth $d$ has at least two
children of depth $\le d-1$, so by induction the subgraph
$\mathcal{G}_v$ rooted at $v$ contains at least $2^{d(v)}$ primitive leaves,
giving $d(v)\le\log_2|\mathcal{V}_p|\le\log_2|\mathcal{V}|$. The
$\mathcal{O}(|\mathrm{Ch}(t^\star)|+|\mathcal{E}|)$ acyclicity check is a
standard DFS reachability test from $\mathrm{Ch}(t^\star)$; the
$\mathcal{O}(1)$ amortized index insertion follows from maintaining
$\mathcal{I}$ as a hash-indexed record.
\end{proof}

% \newpage
% \input{checklist.tex}

\end{document}